\documentclass[journal]{IEEEtran}
\usepackage{booktabs}
\usepackage{amsmath}
\usepackage{array}   
\usepackage[caption=false,font=footnotesize]{subfig} 
\usepackage{url}

\usepackage{cite}
\usepackage{algorithm}  
\usepackage{algorithmicx}
\usepackage{algpseudocode}
\usepackage{color}
\usepackage{graphicx}
\usepackage{float}
\usepackage{subfig}
\usepackage{caption,amssymb}

\newtheorem{theorem}{Theorem}

\newtheorem{proof}{Proof}

\ifCLASSINFOpdf
\else
\fi
\begin{document}
\title{Multi-view Clustering with Deep Matrix Factorization and Global Graph Refinement}
%
%

\author{Chen~Zhang$^{*}$,~Siwei~Wang$^{*}$,~Wenxuan~Tu,~Pei~Zhang, ~Xinwang~Liu$^{\dagger}$,~\IEEEmembership{Senior~Member,~IEEE,}\\
~Changwang~Zhang~ and~Bo~Yuan$^{\dagger}$



}

\maketitle

\begin{abstract}
Multi-view clustering is an important yet challenging task in machine learning and data mining community. One popular strategy for multi-view clustering is matrix factorization which could explore useful feature representations at lower-dimensional space and therefore alleviate dimension curse. However, there are two major drawbacks in the existing work: $i)$ most matrix factorization methods are limited to shadow depth, which leads to the inability to fully discover the rich hidden information of original data. Few deep matrix factorization methods provide a basis for the selection of the new representation's dimensions of different layers. $ii)$ the majority of current approaches only concentrate on the view-shared information and ignore the specific local features in different views. To tackle the above issues, we propose a novel $\mathbf{M}$ulti-$\mathbf{V}$iew $\mathbf{C}$lustering method with $\mathbf{D}$eep semi-N$\mathbf{MF}$ and $\mathbf{G}$lobal $\mathbf{G}$raph $\mathbf{R}$efinement ($\mathbf{MVC}$-$\mathbf{DMF}$-$\mathbf{GGR}$) in this paper. Firstly, we capture new representation matrices for each view by hierarchical decomposition, then learn a common graph by approximating a combination of graphs which are reconstructed from these new representations to refine the new representations in return. An alternate algorithm with proved convergence is then developed to solve the optimization problem and the results on six multi-view benchmarks demonstrate the effectiveness and superiority of our proposed algorithm.
\end{abstract}

\begin{IEEEkeywords}
Multi-view learning, Multi-view clustering, Deep matrix decomposition. 
\end{IEEEkeywords}

%
\IEEEpeerreviewmaketitle

\begin{figure*}
    \centering
	\includegraphics[width = 1\textwidth]{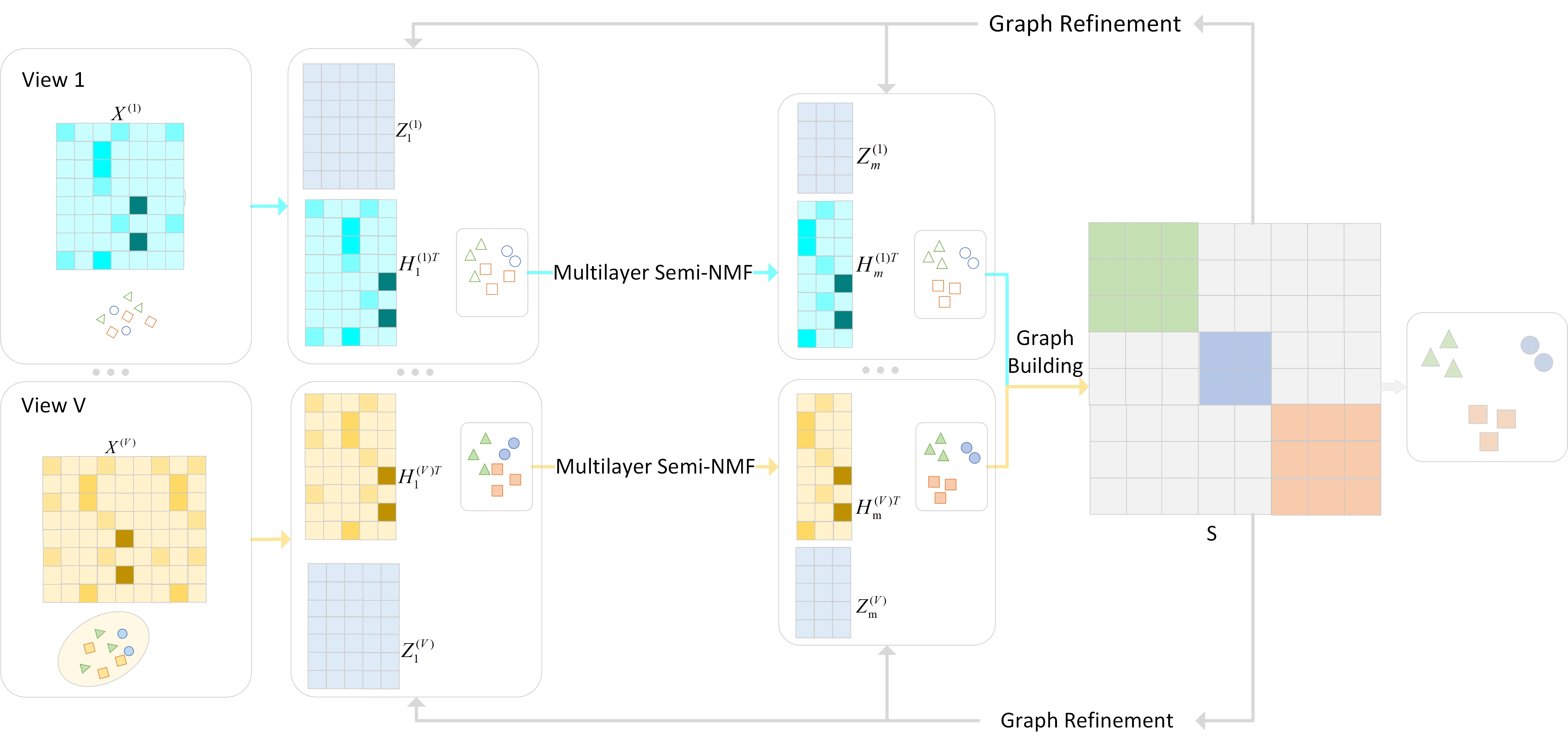}
	\caption{The illustration of our proposed MVC-DMF-GGR. The representations of all layers are firstly initialized with semi-NMF. Then a global graph structure is constructed to capture the complex inner relationship of data. Finally, the representations and the global graph are alternately boosted until best serving for clustering.}
	\label{model structure}
\end{figure*}

\section{Introduction}
\label{Introduction}

\IEEEPARstart{M}ulti-view clustering (MVC) is a fundamental learning task in data mining, image segmentation and pattern recognition \cite{wang2009unified,  fu2010multi,wang2016iterative,tang2019adaptive}. The key of MVC is to find the consistency and complementary information among each view which is described by different aspects, which has been attracted enormous attention. Existing multi-view clustering approaches can be categorized into four categories according to the mechanisms and principles involved, namely, co-training, multi-kernel clustering, graph clustering and subspace clustering \cite{liu2020optimal, cao2015diversity, kang2020partition, zhang2020consensus, wang2020deep}. Co-training algorithms bootstrap the clustering results of different views by using the prior or learning knowledge from others \cite{ kumar_co-training_nodate, kumar_co-regularized_nodate, tan2020unsupervised}. For multi-kernel clustering, the data are first mapped to high-dimensional spaces through kernel functions, and then these kernels are combined linearly or non-linearly to improve clustering performance \cite{li_multiple_nodate, huang2020robust, wang2020kernelized, chen2019jointly}. Multi-view graph clustering algorithms aim to construct graph similarity matrix for individual views, and the main challenge is how to approximately obtain a fusion graph \cite{liang2020multi,wang2019multi,wang2017sparse,wen2020adaptive,wang2015unsupervised,nie_parameter-free_nodate,wang2018multiview,zhang_multilevel_2021}. Multi-view subspace clustering can be further divided into subspace-based methods \cite{wang2015robust,chen_relaxed_2021,ou2020anchor,tang_learning_2019,zhang2020tensorized,zhou_subspace_2020,wang2020learning,tang2018learning} and matrix factorization methods \cite{wang2018multiview,zhu2014fast}. Both of them are designed to learn a low-dimensional representation shared by all views. Our paper belongs to the non-negative matrix factorization method.


In recent years, non-negative matrix factorization (NMF) in Multi-view subspace clustering has been developed to a certain extent. A novel NMF-based multi-view clustering algorithm has been proposed by searching for a factorization that gives compatible clustering solutions across multiple views \cite{liu_multi-view_2013}. The work in \cite{zong_multi-view_2017} proposes a multi-manifold regularized non-negative matrix factorization framework (MMNMF) which can preserve the locally geometrical structure of the manifolds for multi-view clustering. A method of \cite{wang2020structured} aims at multi-view feature selection and fusion problems by using matrix factorization. In \cite{liang2020multi}, a novel NMF model with co-orthogonal constraints is designed to deal with the MVC problem. However, most algorithms based on matrix factorization follow the single-layer strategy. Only a few algorithms such as the method DMVC in \cite{zhao_multi-view_nodate} uses the deep semi-NMF framework inspired by work \cite{trigeorgis_deep_nodate}. DMVC focuses on the intrinsic geometric structure of each view, so graph regularizations are introduced to couple the output representation of deep structures. However, DMVC needs to learn the values of hyperparameters. Then a method \cite{huang_auto-weighted_2020} has been proposed to solve this problem and performance has been further improved. Although these methods have achieved success, they can also be considered to be improved from the following perspectives: $i)$ Since different views represent various attributes of the data items, the view-specific features have been discarded in existing methods and are forced to be consistent among various views. $ii)$ According to \cite{trigeorgis_deep_nodate}, there is still a large gap of fully discovering the rich hidden information of original data with deep factorization matrix structures from existing mechanisms.

In this paper, we propose a multi-view clustering method via deep semi-NMF to solve the above problems. We jointly optimize the representation learning of each view and the late fusion stage in a unified framework, which terms as multi-view clustering with deep semi-NMF and global graph refinement (MVC-DMF-GGR). Firstly, we learn a low-dimensional and more compact representation for each view through the deep semi-NMF framework. As these representations originate from different views, the specific information across views can be well captured. Secondly, we use these learned representations to reconstruct the graph structure of each view and then merge them to approximate a common graph structure. Although the representation of each view may be different, the graph structure of each view tends to be similar. Because they all represent the same batch of samples. Therefore, following traditional graph-based methods, we combine the representation learning and common graph structure learning for joint optimization and hope to obtain an optimal graph structure for clustering. Besides, extensive experiments on six benchmark datasets are performed to evaluate the effectiveness of our proposed method. The proposed method enjoys superior clustering performance by comparing with some state-of-the-art methods.

The contributions of this paper are summarized as follows,
\begin{itemize}
	\item We propose a multi-view clustering method with deep semi-NMF and global graph refinement (MVC-DMF-GGR). In this work, we unify the representation learning and graph structure learning into one framework, which can promote and guide each other and reach a best consensus for clustering. 
	\item Through introducing the deep semi-NMF framework, we decompose the feature matrix by multiple layers and capture the underlying information of each view. In the fusion stage, the graph regularization item is introduced to learn a graph structure shared by each view. The common graph unifies the internal geometric structures of data among different views.  
	\item Extensive experiments are conducted on six multi-view datasets and our proposed method shows clear superiority over other SOTA methods.
\end{itemize}

The rest of the paper is organized as follows. Section \ref{Related Work} outlines the related work of multi-view clustering via NMF. Section \ref{method} introduces the method we have proposed and the alternate algorithm that to solve the optimization problem with its convergence and the computational complexity analysis. Section \ref{experiments} introduces the datasets and compared methods and shows the experiment results with analysis. The ending of this paper is a conclusion in Section \ref{conclusion}.

\section{Related Work}
\label{Related Work}

We introduce some notations firstly. $\mathbf{A}$ represents a matrix which with bold capital symbol. $\mathbf{A}_{i,:}$, $\mathbf{A}_{:,j}$ and $\mathbf{A}_{i,j}$ represent its $i$-th row, $j$-th column and the $ij$-th element. $\mathbf{A}_{m}$ denotes the $m$-th layer and $\mathbf{A}^{(v)}$ denotes the $v$-th view. The Forbenius norm of matrix $\mathbf{A}$ is denoted as $\|\mathbf{A}\|_{F}$ and the trace of matrix $\mathbf{A}$ is denoted as $\operatorname{tr}(\mathbf{A})$. $\mathbf{A}^{\operatorname{T}}$ and $\mathbf{A}^{\dagger}$ denote the transpose and the Moore-Penrose generalized inverse of matrix $\mathbf{A}$ respectively. We separate the negative parts and positive parts of matrix $\mathbf{A}$ as $\left[\mathbf{A}\right]^{+}$ and $\left[\mathbf{A}\right]^{-}$.

In this part, we briefly review several of the most related works, including Semi-NMF, deep Semi-NMF, Multi-view clustering via DMF, etc.
\subsection{Semi-NMF}

Non-negative matrix factorization is an important theme in matrix factorization, which can be used to solve clustering, spectral decomposition, and subspace identification. In reality, the source data $\mathbf{X}$ we get may have mixed signs. The work in \cite{ding_convex_2010} extends traditional NMF to semi-NMF and gives an alternately updating algorithm of related variables. NMF can be written as:
\begin{equation}\label{NMF}
		\mathrm{NMF}: \min_{\mathbf{Z}\geq 0, \mathbf{H}\geq 0} \|\mathbf{X}_{+}-\mathbf{Z}\mathbf{H} \|_F^2,
\end{equation}
Semi-NMF can be written as:
\begin{equation}\label{NMF}
		\mathrm{NMF}: \min_{\mathbf{H}\geq 0} \|\mathbf{X}-\mathbf{Z}\mathbf{H} \|_F^2,
\end{equation}
Where $\mathbf{X} \in \mathbb{R}^{d\times n}$ denotes the input data with $n$ samples and each sample is composed of $d$ dimensional feature. $\mathbf{X}_{+}$ in Eq. (\ref{NMF}) represents the elements of the original data are positive and $\mathbf{X}_{\pm}$ in Eq. (\ref{Semi-NMF}) represents the elements of the original data are mixed. When NMF or semi-NMF is used in clustering, $\mathbf{Z}\in\mathbb{R}^{d\times k}$ is the cluster centroid matrix and $\mathbf{H}\in\mathbb{R}^{k\times n}$ is the soft clustering assignment matrix or the representation of $k$-dimensional. The differences between NMF and Semi-NMF can be concluded that the elements of $\mathbf{X}$ and $\mathbf{Z}$ in NMF are forced to be positive, while in Semi-NMF they can be mix-sign.

The optimization problem of Semi-NMF in Eq. (\ref{Semi-NMF}) can be solved by alternately updating Z and H: $\\ $
 i) \textbf{Optimizing $\mathbf{Z}$ by given $\mathbf{H}$}. By fixing the soft clustering assignment matrix $\mathbf{H}$, the optimization Eq. (\ref{Semi-NMF}) can be considered as an unconstrained problem as:
$\mathcal{C}=\|\mathbf{X}-\mathbf{Z}\mathbf{H}\|_F^2$. By setting $\partial\mathcal{C}/\partial \mathbf{Z}=0$, give the solutions as $\mathbf{Z}=\mathbf{X}\mathbf{H}(\mathbf{H}^{\mathrm{T}}\mathbf{H})^{-1}$.\\
ii) \textbf{Optimizing $\mathbf{H}$ by given $\mathbf{Z}$.} With $\mathbf{Z}$ fixed, $\mathbf{H}$ can be optimized via solving the problem as $\partial\mathcal{C}=\|\mathbf{X}-\mathbf{Z}\mathbf{H}\|_F^2$  with constraint $\mathbf{H} \geq 0$. By using Lagrange Method, we can obtain the update rule of $\mathbf{H}$ which satisfies the KKT condition as follow,
\begin{equation}
	\mathbf{H}=\mathbf{H} \odot \sqrt{(\left[\mathbf{Z}^{\mathrm{T}} \mathbf{X}\right]^{\mathrm{+}}+\left[\mathbf{Z}^{\mathrm{T}}\mathbf{Z}\mathbf{H}\right]^{\mathrm{-}})/(\left[\mathbf{Z}^{\mathrm{T}} \mathbf{X}\right]^{\mathrm{-}}+\left[\mathbf{Z}^{\mathrm{T}} \mathbf{Z} \mathbf{H}\right]^{\mathrm{+}})}.
\end{equation}

\subsection{Deep Semi-NMF for representation learning}\label{deepNMF}

The low-dimensional one-layer representation obtained by Semi-NMF cannot preserve the original feature well due to the limitd representation ability. So a deep Semi-NMF framework for single-view has been proposed in \cite{trigeorgis_deep_nodate}, which is able to learn a lower and hidden representation. This method promotes the applications of semi-NMF and provides interpretability for the improvement of clustering performance. Deep Semi-NMF can be written as,
\begin{equation}\label{Deep_Semi-NMF}
		\min_{\mathbf{H}_{i}\geq 0} \|\mathbf{X}-\mathbf{Z}_{1}\mathbf{Z}_{2}\ldots\mathbf{Z}_{m}\mathbf{H}_{m} \|_F^2, 		\text{s.t.} \mathbf{H}_{i} \geq 0,
\end{equation}
where $\mathbf{Z}_{1}$ denotes the mapping between feature matrix $\mathbf{X}$ and the $1$-th representation $\mathbf{H}_{1}$. $\mathbf{Z}_{i}$ denotes the mapping between the $(i$-$1)$-th representation $\mathbf{H}_{i-1}$ and the $i$-th representation $\mathbf{H}_{i}$. In other words, $\mathbf{H}_{i-1}\approx\mathbf{Z}_{i}\mathbf{H}_{i}$. $m$ denotes the depth of Semi-NMF. Following the work in \cite{ding_convex_2010}, we denote $\Phi=\mathbf{Z}_{1}\mathbf{Z}_{2}\ldots\mathbf{Z}_{m}$. The optimization problem can be solved by alternately updating $\mathbf{Z}$ and $\mathbf{H}$:

i) \textbf{Optimizing $\mathbf{Z}_{i}$ while others be fixed}. The optimization Eq.  (\ref{Deep_Semi-NMF}) can be written as an unconstrained problem as:
$\mathcal{C}=\|\mathbf{X}-\Phi\mathbf{Z}_{i}\mathbf{H}_{i}\|_F^2$. By setting $\partial\mathcal{C}/\partial \mathbf{Z}_{i}=0$, we can give the solution as $ \mathbf{Z}_{i}=\Phi^{\dagger} \mathbf{X} \mathbf{H}_{i}^{\dagger}$.

ii) \textbf{Optimizing $\mathbf{H}_{i}$ while others be fixed}. $\mathbf{H}_{i}$ can be optimized via solving the problem as $\mathcal{C}=\|\mathbf{X}-\Phi\mathbf{H}_{i}\|_F^2$ with constraintion $\mathbf{H}_{i} \geq 0$. By using Lagrange method, we can obtain the update rule of $\mathbf{H}_{i}$ which satisfies the KKT condition as follow,
\begin{equation}
	\mathbf{H}_{i}=\mathbf{H}_{i} \odot \sqrt{(\left[\Phi^{\mathrm{T}} \mathbf{X}\right]^{\mathrm{+}}+\left[\Phi^{\mathrm{T} }\Phi\mathbf{H}\right]^{\mathrm{-}})/(\left[\Phi^{\mathrm{T}} \mathbf{X}\right]^{\mathrm{-}}+\left[\Phi^{\mathrm{T}} \Phi \mathbf{H}\right]^{\mathrm{+}})}.
\end{equation}

\subsection{DMVC}
The work of \cite{zhao_multi-view_nodate} combines deep semi-NMF with multi-view clustering which is called DMVC. The proposed method solves the clustering problem with constant geometric structure and representation learning by multi-layer simultaneously. Formally, multi-view clustering with deep semi-NMF can be mathematically written as,
\begin{equation}\label{MCV_DMF}
    \begin{split}
		\min\limits_{\substack{\mathbf{Z}_{i}^{(v)},\mathbf{H}_{i}^{(v)}\\\mathbf{H}_{m},\mathbf{\alpha}^{(v)}}}\sum_{v=1}^{V}(\alpha^{(v)})^\gamma(&\|\mathbf{X}^{(v)}-\mathbf{Z}_{1}^{(v)}\mathbf{Z}_{2}^{(v)}\ldots\mathbf{Z}_{m}^{(v)}\mathbf{H}_{m}\|_F^2\\&+\mathbf{\beta}\operatorname{tr}( \mathbf{H}_{m}\mathbf{L}^{(v)}  \mathbf{H}_{m}^{\mathrm{T}})),\\
		\text{s.t.} \mathbf{H}_{i}^{(v)}\geq 0,&\mathbf{H}_{m}\geq 0,\sum_{v=1}^{V}\alpha^{(v)}=1, \alpha^{(v)} \geq 0.
	\end{split}
\end{equation} 
We denote $\mathbf{X}=(\mathbf{X}^{(1)}\ldots\mathbf{X}^{(v)}\ldots\mathbf{X}^{(V)})$, where $\mathbf{X}^{(v)}\in\mathbb{R}^{d_v\times n}$ represents the feature matrix of the $v$-th view. $n$ and $d_v$ denote $\mathbf{X}^{(v)}$ with $n$ sample and each sample is of $d_v$ dimensional feature. Similar to subsection \ref{deepNMF}, $\mathbf{Z}_{1}^{(v)}$ denotes the mapping between feature matrix $\mathbf{X}^{(v)}$ and the $1$-th representation $\mathbf{H}_{1}^{(v)}$ of the $v$-th view. $\mathbf{Z}_{i}^{(v)}$ denotes the mapping between the $(i$-${1})$-th representation $\mathbf{H}_{i-1}^{(v)}$ and the $i$-th representation $\mathbf{H}_{i}^{(v)}$ of the $v$-th view$v$. $V$ is the number of views and $m$ is the number of layers or called the depth of Semi-NMF. $\mathbf{H}_m$ is the consensus latent representation for all views. $\alpha^{(v)}$ is the weighting coefficient of the $v$-th view and $\gamma$ is a coefficient that controls the weights distribution.  $\mathbf{L}^{(v)}=\mathbf{D}^{(v)}-\mathbf{A}^{(v)}$ denotes the $v$-th graph Laplacian, where $\mathbf{A}^{(v)}$ is constructed by feature matrix using $k$-nearest neighbor and $\mathbf{D}_{ii}^{(v)}=\sum_{j}\mathbf{A}_{ij}^{(v)}$.
The optimization problem of Eq. (\ref{MCV_DMF}) can be solved by alternately updating $\mathbf{Z}_{i}^{(v)}$, $\mathbf{H}_{i}^{(v)}$, $\mathbf{H}_{m}$ and $\mathbf{\alpha}^{(v)}$. The update rule of $\mathbf{Z}_{i}^{(v)}$, $\mathbf{H}_{i}^{(v)}$ and $\mathbf{H}_{m}$ are similar to the method deep semi-NMF. As for updating $\mathbf{\alpha}^{(v)}$, we can use Lagrange method and take the derivative of Lagrange function with respect to $\mathbf{\alpha}^{(v)}$.

\section{The proposed method}
\label{method}
\begin{table}[t]
    \renewcommand{\arraystretch}{1.3}
	\caption{{Basic notations for the proposed method.}}
	\label{notions} 
	\begin{tabular}{ll}
		\toprule
		Notations       & Meaning \\
		\midrule
		$\mathbf{X}^{(v)} \in \mathbb{R}^{d_{v}\times n}$   			&Feature matrix of the $i$-th view \\
		$\mathbf{Z}_{1}^{(v)} \in \mathbb{R}^{d_{v}\times l_{i}}$		&$1$-th layer cluster centroid matrix of the $i$-th view \\
		$\mathbf{Z}_{i}^{(v)} \in \mathbb{R}^{l_{i-1}\times l_{i}}$	&$i$-th layer cluster centroid matrix of the $i$-th view \\
		$\mathbf{H}_{i}^{(v)} \in \mathbb{R}^{l_{i}\times n}$			& $i$-th layer feature representation of the $i$-th view \\
		$\mathbf{S}^{(v)} \in \mathbb{R}^{n\times n}$							&the similarity matrix of the $v$-th view\\
		$\mathbf{S} \in \mathbb{R}^{n\times n}$							&Consensus similarity matrix \\
		$\phi  \in \mathbb{R}^{d_{v}\times l_{i-1}}$						&$ \mathbf{Z}_{1}^{(v)}\mathbf{Z}_{2}^{(v)}\ldots\mathbf{Z}_{i-1}^{(v)}$\\
		$\Phi  \in \mathbb{R}^{d_{v}\times l_{i}}$						&$\mathbf{Z}_{1}^{(v)}\mathbf{Z}_{2}^{(v)}\ldots\mathbf{Z}_{i}^{(v)}$ \\
		$\mathbf{\hat{H}}_{i}^{(v)} \in \mathbb{R}^{l_{i}\times n}$	&$\mathbf{Z}_{i+1}^{(v)}\dots\mathbf{Z}_{m}^{(v)}\mathbf{H}_{m}^{(v)}$\\
		$\mathbf{G}  \in \mathbb{R}^{n\times n}$     					&$\sum_{o=1,o\neq v}^{V}\alpha^{(o)}\mathbf{H_{m}}^{(o)\mathrm{T}}{\mathbf{H_{m}}^{(o)}}$\\
		$\mathbf{Q} \in \mathbb{R}^{n\times n}$  						&$\sum_{v=1}^{V}\alpha^{(v)}\mathbf{H}_{m}^{(v)\mathrm{T}}{\mathbf{H}_{m}^{(v)}}$ \\	\bottomrule
	\end{tabular}
\end{table}
We introduce some basic notations of our method firstly as described in Table \ref{notions}. We also explain in the relevant places of the paper for reading easily. 

As we mentioned before, the representations of all views in the last layer should be different in theory and the global graph structure which represents the relationship between samples should be consistent. Therefore, different from DMVC, we assume that the feature representations of the last layer in different views are different and a consensus local structure matrix $\mathbf{S}$ should be fused with individual structures. The idea can be mathematically expressed as follows,
\begin{equation}\label{object_all}
	\begin{aligned}
		\min \limits_{\substack{\mathbf{Z}_{i}^{(v)}, \mathbf{H}_{i}^{(v)}\\\mathbf{\alpha}^{(v)}, \mathbf{S}}}
		&\sum_{v=1}^{V}\|\mathbf{X}^{(v)}-\mathbf{Z}_{1}^{(v)}\mathbf{Z}_{2}^{(v)}\ldots\mathbf{Z}_{m}^{(v)}\mathbf{H}_{m}^{(v)}\|_F^2	\\	&+\mathbf{\beta}\|\mathbf{S}-\sum_{v=1}^{V}\alpha^{(v)}\mathbf{H}_{m}^{(v)\mathrm{T}}{\mathbf{H}_{m}^{(v)}}\|_F^2, \\
		\text{s.t.} \mathbf{H}_{i}^{(v)}\geq 0, &\sum_{v=1}^{V}\alpha^{(v)}\!=\!1, \alpha^{(v)} \!\geq\! 0, \mathbf{S}\mathbf{1}\!=\!\mathbf{1},\mathbf{S}\! \geq \!0,\text{diag}\mathbf{(S)}\!=\!0.
	\end{aligned}
\end{equation}
The meaning of $\mathbf{X}^{(v)}$, $\mathbf{Z}_{i}^{(v)}$ and $\mathbf{H}_{i}^{(v)}$ are similar to these symbols in Eq. (\ref{MCV_DMF}) described in Table \ref{notions}. $\mathbf{H}_{m}^{(v)}$ denotes the $m$-th layer of the $v$-th view. $\mathbf{H}_{m}^{(v)\operatorname{T}}\mathbf{H}_{m}^{(v)}$ constructs the similarity matrix $\mathbf{S}^{(v)}$ in different layer. $\alpha^{(v)}$ is the weight coefficient of the $v$-th view for $\mathbf{S}^{(v)}$. $\mathbf{S}$ denotes the consensus similarity matrix. $\mathbf{S}_{i,j}$ denotes the similarity score between $i$-th and $j$-th sample so we need to add the constraints $\mathbf{S}\geq 0$ and $\text{diag}\mathbf{(S)}=0$ for $\mathbf{S}$. The larger value $\mathbf{S}_{i,j}$ is, the more likely two samples belong to the same cluster. We hope to obtain normalized solution, so we add the constraint $\mathbf{S}\mathbf{1}=\mathbf{1}$.

\subsection{Initialization}\label{init}
Inspired by the tricks of the initialization in \cite{hinton2006reducing}, we have pre-trained all of the layers to initialize the variables $\mathbf{Z}_{i}^{(v)}$ and $\mathbf{H}_{i}^{(v)}$ by decomposing layer by layer. Firstly, we decompose the feature matrix of the $v$-th view $\mathbf{X}^{(v)}\approx \mathbf {Z}_{1}^{(v)} \mathbf{H}_{1}^{(v)}$, where $\mathbf{Z}_{1}^{(v)}\in\mathbb{R}^{d_v\times l_1}$ and $\mathbf{H}_{1}^{(v)}\in\mathbb{R}^{l_1\times n}$. Following this, we decompose the new feature matrix $\mathbf{H}_{1}^{(v)}\approx\mathbf{Z}_{2}^{(v)} \mathbf{H}_{2}^{(v)}$, where $\mathbf{Z}_{2}^{(v)}\in\mathbb{R}^{l_{1}\times l_{2}}$ and $\mathbf{H}_{2}^{(v)}\in\mathbb{R}^{l_{2}\times n}$. We repeat the above steps until all layers have been pre-trained.  We pre-train each of the layers to have an initial approximation of the matrices $\mathbf{H}_{i}^{(v)}$ and $\mathbf{Z}_{i}^{(v)}$ which can greatly reduce the time for follow-up work. Then we use the value of $\mathbf{Z}_{i}^{(v)}$ and $\mathbf{H}_{i}^{(v)}$ to initialize $\mathbf{S}$ and $\alpha^{(v)}$ by setting $\alpha^{(v)} = \frac{1}{V}$ and $\mathbf{S}=\sum_{v=1}^{V}\alpha^{(v)}\mathbf{H}_{m}^{(v)\mathrm{T}}\mathbf{H}_{m}^{(v)}$. At the beginning, we argue that each view has the same contribution, so we initialize $\mathbf{S}$ by the construction of $\mathbf{H}_m^{(v)}$ with the same weight.

\subsection{Optimization}\label{sec:Optimization}
Because the objective function Eq. ($\ref{object_all}$) is a non-convex problem, it seems unlikely to solve this problem in one step. So we propose a five-step alternate optimization method to address this problem. To reduce the total reconstruction error of the model, we also need to alternately minimize  $\mathbf{Z_{i}}^{(v)}$ and $\mathbf{H_{i}}^{(v)}$ in each layer.
\subsubsection{Update rule for matrix $\mathbf{Z}_{i}^{(v)}$}	
By fixing $\mathbf{H}_{i}^{(v)}$, $\mathbf{S}$, $\alpha^{(v)}$ and $\mathbf{Z}_{i}^{(o)}$($o\neq v$), we can update $\mathbf{Z}_{i}^{(v)}$ by solving the following problem without constraint, 
\begin{equation}\label{upZ}  
	\begin{array}{l}
		\min\limits_{\mathbf{Z}_{i}^{(v)}} \|\mathbf{X}^{(v)}-\phi\mathbf{Z}_{i}^{(v)}\hat{\mathbf{H}}_{i}^{(v)}\|_F^2,
	\end{array}
\end{equation}
where $\phi=\mathbf{Z}_{1}^{(v)}\mathbf{Z}_{2}^{(v)}\ldots\mathbf{Z}_{i-1}^{(v)}$, by setting $\partial\mathcal{C}/\partial \mathbf{Z}_{i}^{(v)}=0$, we can give the solutions as,
\begin{equation}
	\mathbf{Z}_{i}^{(v)}=\phi^{\dagger} \mathbf{X}^{(v)}\mathbf{\hat{H}}_{i}^{(v)\dagger},
\end{equation}	
where $\phi^{\dagger}=(\phi^{\mathrm{T}} \phi)^{-1} \phi^{\mathrm{T}}$ and $\mathbf{\hat{H}}_{i}^{(v) \dagger}=\mathbf{\hat{H}}_{i}^{(v)\mathrm{T}}(\mathbf{\hat{H}}_{i}^{(v)} \mathbf{\hat{H}}_{i}^{(v) \mathrm{T}})^{-1}$. $\mathbf{\hat{H}}_{i}^{(v)}=\mathbf{Z}_{i+1}^{(v)}\dots\mathbf{Z}_{m}^{(v)}\mathbf{H}_{m}^{(v)}$ and $\mathbf{\hat{H}}_{i}^{(v)}$ denotes the reconstruction of the $i$-th layer's representation for the $v$-th view.

\subsubsection{Update rule for matrix $\mathbf{H}_{i}^{(v)}(i \textless m)$}	
By fixing $\mathbf{Z}_{i}^{(v)}$, $\mathbf{H}_{m}^{(v)}$, $\alpha^{(v)}$ and $\mathbf{S}$, we can update $\mathbf{H}_{i}^{(v)}$ by solving the following problem, 
\begin{equation}\label{Hi}
	\min\limits_{\mathbf{H_{i}}^{(v)}} \|\mathbf{X}^{(v)}-\Phi\mathbf{H}_{i}^{(v)}\|_F^2,  \text{s.t.} \mathbf{H}_{i}^{(v)} \geq 0,
\end{equation}
where $\Phi=\mathbf{Z}_{1}^{(v)}\mathbf{Z}_{2}^{(v)}\ldots\mathbf{Z}_{i}^{(v)}$. Following the update rule in \cite{ding_convex_2010}, the update rule for $\mathbf{H}_{i}^{(v)}(i \textless m)$ can be written as,
\begin{equation}\label{reHi}
	\mathbf{H}_{i}^{(v)}=\mathbf{H}_{i}^{(v)} \odot \sqrt{\frac{\left[\Phi^{\mathrm{T}} \mathbf{X}^{(v)}\right]^{\mathrm{+}}+\left[\Phi^{\mathrm{T}} \Phi \mathbf{H}_{i}^{(v)}\right]^{\mathrm{-}}}{\left[\Phi^{\mathrm{T}} \mathbf{X}^{(v)}\right]^{\mathrm{-}}+\left[\Phi^{\mathrm{T}} \Phi \mathbf{H}_{i}^{(v)}\right]^{\mathrm{+}}}}.
\end{equation}
We also update $\mathbf{H}_{m}^{(v)}$ here for faster convergence and easier code writing.

\subsubsection{Update rule for matrix $\mathbf{H}_{m}^{(v)}$ }
By fixing $\mathbf{Z}_{i}^{(v)}$, $\mathbf{H}_{i}^{(v)}(i \textless m)$, $\alpha^{(v)}$ and $\mathbf{S}$, we can update $\mathbf{H}_{m}^{(v)}$ by solving the following problem, 
\begin{equation}\label{Hm}
	\min\limits_{\mathbf{H_{m}}^{(v)}}\|\mathbf{X}^{(v)}-\Phi\mathbf{H}_{m}^{(v)}\|_F^2
	+\beta\|\mathbf{S}-\alpha^{(v)}\mathbf{H}_{m}^{(v)\mathrm{T}}{\mathbf{H}_{m}^{(v)}}-\mathbf{G}\|_F^2,\text{s.t.} \mathbf{H}_{m}^{(v)} \geq 0,
\end{equation}
where the variables are defined as follows,
\begin{equation}\label{var1}
    \Phi=\mathbf{Z}_{1}^{(v)}\mathbf{Z}_{2}^{(v)}\ldots\mathbf{Z}_{m}^{(v)},\\
    \mathbf{G}=\sum_{o=1,o\neq v}^{V}\alpha^{(o)}\mathbf{H_{m}}^{(o)\mathrm{T}}{\mathbf{H_{m}}^{(o)}}.
\end{equation}
We give the updating rule of $\mathbf{H}_{m}^{(v)}$ firstly, followed by the proof of it.
\begin{equation}\label{reHm}
	\begin{split}
		&\mathbf{H}_{m}^{(v)}=\mathbf{H}_{m}^{(v)} \odot \sqrt{
			\vartheta_{u}(\mathbf{ZHS})/
			\vartheta_{l}(\mathbf{ZHS})},\\
		&\vartheta_{u}(\mathbf{ZHS})=\left[\Phi^{\mathrm{T}} \mathbf{X}^{(v)}\right]^{\mathrm{+}}+\left[\Phi^{\mathrm{T}}\Phi\mathbf{H}_{m}^{(v)}\right]^{\mathrm{-}}+
		\alpha^{(v)}\beta(\left[\mathbf{H}_{m}^{(v)}\mathbf{S}\right]^{\mathrm{+}}\\
		&+\left[\mathbf{H}_{m}^{(v)}\mathbf{S}^{\mathrm{T}}\right]^{\mathrm{+}}+\left[2\mathbf{H}_{m}^{(v)}\mathbf{G}\right]^{\mathrm{-}}+
		\left[2{\alpha^{(v)}}\mathbf{H}_{m}^{(v)}\mathbf{H}_{m}^{(v)\mathrm{T}}\mathbf{H}_{m}^{(v)}\right]^{\mathrm{-}}),\\
		&\vartheta_{u}(\mathbf{ZHS})=\left[\Phi^{\mathrm{T}} \mathbf{X}^{(v)}\right]^{\mathrm{-}}+\left[\Phi^{\mathrm{T}}\Phi\mathbf{H}_{m}^{(v)}\right]^{\mathrm{+}}+
		\alpha^{(v)}\beta(\left[\mathbf{H}_{m}^{(v)}\mathbf{S}\right]^{\mathrm{-}}\\
		&+\left[\mathbf{H}_{m}^{(v)}\mathbf{S}^{\mathrm{T}}\right]^{\mathrm{-}}+\left[2\mathbf{H}_{m}^{(v)}\mathbf{G}\right]^{\mathrm{+}}+
		\left[2{\alpha^{(v)}}\mathbf{H}_{m}^{(v)}\mathbf{H}_{m}^{(v)\mathrm{T}}\mathbf{H}_{m}^{(v)}\right]^{\mathrm{+}}).
	\end{split}
\end{equation}

\begin{theorem}
	The limited solution of the update rule in Eq. (\ref{reHm}) satisfies the KKT condition.
\end{theorem}

\begin{proof}
	We introduce the Lagrangian function as
	\begin{equation}	
		\begin{split}
			\mathbf{L}(\mathbf{H}_{m}^{(v)})=&\|\mathbf{X}^{(v)}-\Phi\mathbf{H}_{m}^{(v)}\|_F^2+\beta\|\mathbf{S}-\alpha^{(v)}\mathbf{H}_{m}^{(v)\mathrm{T}}\mathbf{H}_{m}^{(v)}-\mathbf{G}\|_F^2\\&-\operatorname{Tr}(\eta\mathbf{H}_{m}^{(v)}),
		\end{split}
	\end{equation}
	In order to satisfy the constraint $\mathbf{H}_{m}^{(v)}\geq0$, we introduce the Lagrangian multiplier $\eta$. By setting $\partial\mathbf{L}(\mathbf{H}_{m}^{(v)})/\partial \mathbf{H}_{m}^{(v)}=0$, we can obtation:
	\begin{equation}
		\begin{split}
			&\partial \mathbf{L}(\mathbf{H}_{m}^{(v)})/\partial \mathbf{H}_{m}^{(v)} =-2( \Phi^\mathrm{T} \mathbf{X}^{(v)}-\Phi^\mathrm{T} \Phi \mathbf{H}_{m}^{(v)}+\alpha^{(v)}\beta\mathbf{H}_{m}^{(v)}\mathbf{S}\\&+\alpha^{(v)}\beta\mathbf{H}_{m}^{(v)}\mathbf{S}^{\mathrm{T}}-2\alpha^{(v)}\beta\mathbf{H}_{m}^{(v)}\mathbf{G}-2\alpha^{(v)2}\beta\mathbf{H}_{m}^{(v)}\mathbf{H}_{m}^{(v)\mathrm{T}}\mathbf{H}_{m}^{(v)})\\&-\eta=0.
		\end{split}
	\end{equation}
	From the complementary slackness condition, we can obtain,
	\begin{equation}\label{proofHm1}
		\begin{split}
			&(-\Phi^\mathrm{T}\mathbf{X}^{(v)}+\Phi^\mathrm{T}\Phi\mathbf{H}_{m}^{(v)}-\alpha^{(v)}\beta\mathbf{H}_{m}^{(v)}\mathbf{S}-\alpha^{(v)}\beta\mathbf{H}_{m}^{(v)}\mathbf{S}^{\mathrm{T}}+\\
			&2\alpha^{(v)}\beta\mathbf{H}_{m}^{(v)}\mathbf{G}+2\alpha^{(v)2}\beta\mathbf{H}_{m}^{(v)}\mathbf{H}_{m}^{(v)\mathrm{T}}\mathbf{H}_{m}^{(v)})\mathbf{H}_{m}^{(v)}=\eta\mathbf{H}_{m}^{(v)}=0
		\end{split}
	\end{equation}
	Both the equations require that at least one of the two factors is equal to zero, so Eq. (\ref{proofHm1}) and Eq. (\ref{proofHm2}) have the same meaning. We multiply both sides by $\mathbf{H}_{m}^{(v)}$ and we can obtain:
	\begin{equation}\label{proofHm2}
		\begin{split}
			&(-\Phi^\mathrm{T}\mathbf{X}^{(v)}+\Phi^\mathrm{T}\Phi\mathbf{H}_{m}^{(v)}-\alpha^{(v)}\beta\mathbf{H}_{m}^{(v)}\mathbf{S}-\alpha^{(v)}\beta\mathbf{H}_{m}^{(v)}\mathbf{S}^{\mathrm{T}}+\\
			&2\alpha^{(v)}\beta\mathbf{H}_{m}^{(v)}\mathbf{G}+2\alpha^{(v)2}\beta\mathbf{H}_{m}^{(v)}\mathbf{H}_{m}^{(v)\mathrm{T}}\mathbf{H}_{m}^{(v)})\mathbf{H}_{m}^{(v)2}=0.
		\end{split}
	\end{equation}
	Eq. (\ref{proofHm2}) is a fixed point equation. By noting that $\Phi^\mathrm{T}\mathbf{X}^{(v)}=\left[\Phi^\mathrm{T}\mathbf{X}^{(v)}\right]^{+}-\left[\Phi^\mathrm{T}\mathbf{X}^{(v)}\right]^{-}$, $\Phi^\mathrm{T}\Phi\mathbf{H}_{m}^{(v)}=\left[\Phi^\mathrm{T}\Phi\mathbf{H}_{m}^{(v)}\right]^{+}-\left[\Phi^\mathrm{T}\Phi\mathbf{H}_{m}^{(v)}\right]^{-}$ etc, it is easy to get the update rule Eq. (\ref{reHm}) for $\mathbf{H}_{m}^{(v)}$ and to see that the equation satisfies the fixed point equation. At convergence, $\mathbf{H}_{m}^{(v)(\infty)}=\mathbf{H}_{m}^{(v)(t+1)}=\mathbf{H}_{m}^{(v)(t)}=\mathbf{H}_{m}^{(v)}$.
\end{proof}

\subsubsection{Update rule for matrix $\mathbf{S}$}	
By fixing $\mathbf{Z}_{i}^{(v)}$, $\mathbf{H}_{i}^{(v)}$ and $\alpha^{(v)}$, we can update $\mathbf{S}$ by solving the following problem, 
\begin{equation}\label{update S}
	\min\limits_{\mathbf{S}}\|\mathbf{S}-\mathbf{Q}\|_F^2,
	\text{s.t.}\mathbf{S}\mathbf{1}=\mathbf{1},\mathbf{S}\geq 0,\text{diag}\mathbf{(S)}=0,
\end{equation}
where $\mathbf{Q}=\sum_{v=1}^{V}\alpha^{(v)}\mathbf{H}_{m}^{(v)\mathrm{T}}{\mathbf{H}_{m}^{(v)}}$. This problem yields a close-formed solution that,
\begin{equation}\label{reS}
	\mathbf{S}_{i,:}=\max \left({\mathbf{Q}}_{i,:}+\gamma\mathbf{1}^\mathrm{T}, 0\right), \mathbf{S}_{i i}=0, \gamma=\frac{1-{\mathbf{Q}}_{i,:} \mathbf{1}}{n},
\end{equation}
where $\mathbf{S}_{i,:}$ is the $i$-th row of $\mathbf{S}$, ${\mathbf{Q}}_{i,:}$ is the $i$-th row of ${\mathbf{Q}}$.

\begin{theorem}
	Eq. (\ref{reS}) is the close-formed solution of Eq. (\ref{update S}).
\end{theorem}

\begin{proof} 
	The problem of Eq. (\ref{update S}) can be easily rewritten into $n$ row-formed independent optimization problems as follow,
	\begin{equation}\label{update S2}
		\min _{{S}_{i,:}}\left\|\mathbf{S}_{i,:}-{\mathbf{Q}}_{i,:}\right\|_{\mathrm{F}}^{2}, 
		\text{s.t.}\mathbf{S}_{i,:} \geq 0, \mathbf{S}_{i,:}\mathbf{1}=1, \mathbf{S}_{i, i}=0,
	\end{equation}
	The Lagrangian function of Eq. (\ref{update S2}) is,
	\begin{equation}\label{lagrangianS}
		\mathcal{L}\left(\mathbf{S}_{i,:}, \gamma, \eta\right)=\left\|\mathbf{S}_{i,:}-{\mathbf{Q}}_{i,:}\right\|_{F}^{2}-\gamma\left(\mathbf{S}_{i,:}\mathbf{1}-1\right)-\eta\mathbf{S}_{i,:}^\mathrm{T},
	\end{equation}
	where $\gamma$ and $\eta$ are the Lagrangian multipliers for the constraints $\mathbf{S}_{i,:} \geq 0$ and $\mathbf{S}_{i,:}\mathbf{1}=1$ respectively. Then the KKT condition is written as,
	\begin{equation}
		\left\{
		\begin{array}{lr}
			\mathbf{S}_{i,:}-{\mathbf{Q}}_{i,:}-\gamma\mathbf{1}^\mathrm{T}-\eta=0, \\
			\eta \odot \mathbf{S}_{i,:}^\mathrm{T}=0,
		\end{array}
		\right.
	\end{equation}
	We can easily obtain the Eq. (\ref{reS}).
\end{proof}

\subsubsection{Update rule for coefficient $\alpha$}
By fixing $\mathbf{Z}_{i}^{(v)}$, $\mathbf{H}_{i}^{(v)}$ and $\mathbf{S}$, we can update $\alpha$ by solving the following problem, 
\begin{equation}\label{update alpha1}
	\min\limits_{\alpha}\|\mathbf{S}-\sum_{v=1}^{V}\alpha^{(v)}\mathbf{H}_{m}^{(v)\mathrm{T}}\mathbf{H}_{m}^{(v)}\|_F^2,
	\text{s.t.} \sum_{v=1}^{V}\alpha^{(v)}=1, \alpha^{(v)} \geq 0.
\end{equation}
Supposing $\mathbf{Q}=\sum_{v=1}^{V}\alpha^{(v)}\mathbf{H_{m}}^{(v)\mathrm{T}}{\mathbf{H_{m}}^{(v)}}$, we have that,
\begin{equation}\label{update_alpha2}
	\|\mathbf{S}-\mathbf{Q}\|_F^2=n-2\operatorname{Tr}\left(\mathbf{S}^{\top}\mathbf{Q}\right)+\operatorname{Tr}\left(\mathbf{Q}^{\top}\mathbf{Q}\right).
\end{equation}

Note that $\mathbf{Q}$=$\sum_{v=1}^{V}\alpha^{(v)}\mathbf{H_{m}}^{(v)\mathrm{T}}{\mathbf{H_{m}}^{(v)}}$, we have $\mathbf{Q}^{\operatorname{T}}$=$\mathbf{Q}$. Taking them into Eq. (\ref{update_alpha2}), the optimization can be written as follows,
\begin{equation}\label{re_alpha}
	\min _{\alpha} \frac{1}{2} \alpha^\mathrm{T} \mathbf{A} \alpha-\mathbf{f}^\mathrm{T}\alpha,
	\text{s.t.} \alpha^\mathrm{T} \mathbf{1}=1, \quad \alpha \geq 0,
\end{equation}    
where the variables are defined as follows,
\begin{equation}\label{var2}
    \begin{aligned}
        &\mathbf{f}^{\mathrm{T}}=\left[f_{1}, f_{2}, \ldots, f_{V}\right],\\
        &\mathbf{f}_{v}=\operatorname{Tr}\left(\mathbf{S}^\mathrm{T}\mathbf{H}_{m}^{(v)\mathrm{T}} \mathbf{H}_{m}^{(v)}\right),\\
        &\mathbf{A}_{p q}=\operatorname{Tr}\left(\mathbf{H}_{m}^{(p)\mathrm{T}}\mathbf{H}_{m}^{(p)}\mathbf{H}_{m}^{(q)\mathrm{T}}\mathbf{H}_{m}^{(q)}\right),\\
        &\alpha=\left[\alpha^{(1)},\alpha^{(2)},\ldots,\alpha^{(V)}\right]^{\mathrm{T}}.
    \end{aligned}
\end{equation}    

For every $x \in \mathbb{R}^{m}$, we have that $x^\mathrm{T}\mathbf{A}x=\left\|\sum_{v=1}^{V} x_{v}\mathbf{H}_{m}^{(v)}\mathbf{H}_{m}^{(v)\mathrm{T}}\right\|_{F}^{2} \geq 0$. So the matrix $\mathbf{A}$ is a positive semi-definite matrix and quadratic programming could be used in Eq. (\ref{re_alpha}).

The entire approach is outlined in Algorithm \ref{Algorithm-proposed}. We train the proposed algorithm at least 150 iterations until convergence, then we perform spectral clustering on $\mathbf{S}$ to obtain the clustering results.

\alglanguage{pseudocode}
\begin{algorithm}[t]
	\caption{MVC-DMF-GGR}
	\label{Algorithm-proposed}
	\hspace*{0.02in}{$\mathbf{Input:}$}  \hspace*{0.02in}
	Set of given Multi-view data matrices $\mathbf{X}^{(v)}(1\leq v \leq V)$, tuning parameters $\beta$ and the depth of layers $p$.\\
	\hspace*{0.02in}{$\mathbf{Initialize:}$} Initialize $\mathbf{H}_{i}^{(v)}$ and $\mathbf{Z}_{i}^{(v)}$ according to \ref{init}, and then initialize $\alpha^{(v)}$ and  $\mathbf{S}$.\\
	\hspace*{0.02in}{$\mathbf{Output:}$} Performing spectral clustering on \textbf{S}.
	
	\begin{algorithmic}[1]
		\While{not\;convergence}
			\For{i $\leq$ m}
				\State update $\mathbf{Z}_{i}^{(v)}$ by solving Eq. (\ref{upZ}).
				\State update $\mathbf{H}_{i}^{(v)}$ by solving Eq. (\ref{reHi}).
			\EndFor
    		\State update $\mathbf{H}_{m}^{(v)}$ by solving Eq. (\ref{reHm}).
    		\State update $\mathbf{S}$ by solving Eq. (\ref{reS}).
    		\State update $\alpha$ by solving Eq. (\ref{re_alpha}).
		\EndWhile 
		\State \Return Similarity matrix $\mathbf{S}.$ Performing spectral clustering on $\mathbf{S}$ to get final clustering partition.
	\end{algorithmic}
\end{algorithm}

\subsection{Analysis and discussions}  

$\textit{Computational Complexity}$: Pre-training and fine-tuning are the two main stages of our proposed method, and we will analyze them
separately. To make the analysis clearer, we assume the dimensions in all the layers are the same. So we denote $l$ and the dimensions of the original feature for all the views are the same which denoted $d$. $t_{pre}$ denotes the number of iterations to achieve convergence in pre-training process and $t_{fine}$ denotes the number of iterations to achieve convergence in fine-tuning process. So the complexity of pre-training and fine-tuning stages are $O(Vm{t_{pre}}(dnl+nl^{2}+ld^{2}+ln^{2}+dn^{2}))$ and $O(Vm{t_{fine}}(dnl+nl^{2}+ld^{2}+ln^{2}+dn^{2}))$ respectively, where $l \le d$ normally. In conclusion, the time complexity of our algorithm is $O(Vm(dnl+ld^{2}+dn^{2})(t_{pre}+t_{fine}))$.

$\textit{Convergence}$: It is easy to obtain that the lower bound of the whole optimization function is 0. When we optimize one variable with fixing the others, the four (optimizing $\mathbf{Z}_{i}^{(v)}$ and $\mathbf{H}_{i}^{(v)}$ as one subproblem) subproblems are strictly convex and the objective of Algorithm \ref{Algorithm-proposed} is monotonically decreased at each iteration. As a result, the proposed algorithm can be confirmed to be convergent.

\begin{figure*}
	
	\begin{minipage}{\linewidth}
		\vspace{3pt}
		\centerline{\includegraphics[width=\textwidth]{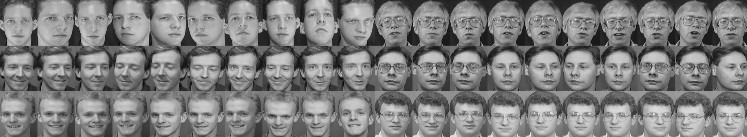}}
		\centerline{(a) ORL}
	\end{minipage}
	\begin{minipage}{\linewidth}
		\vspace{3pt}
		\centerline{\includegraphics[width=\textwidth]{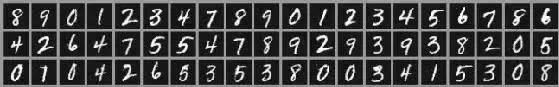}}
		\centerline{(b) HW}
	\end{minipage}
	\caption{Sample images from ORL and HW.}
	\label{samples}
\end{figure*}

\begin{table}[t]
	\caption{{Datasets used in our experiments.}}\label{datasets} 
	\centering
	\begin{tabular}{ccccc}
		\toprule
		Dataset       & Views & Samples & Classes & Datatypes \\
		\midrule
		HW   & 2     & 2000    & 10               & Image     \\
		BBCSport      & 2     & 544     & 5                & Text      \\
		3Sources      & 3     & 169     & 6                & Text      \\
		BBC           & 4     & 685     & 5                & Text      \\
		CiteSeer      & 2     & 3312    & 6                & Text      \\
		ORL           & 3     & 400     & 40               & Image     \\
		\bottomrule
	\end{tabular}
\end{table}

\section{Experiments}
\label{experiments}
In this part, we evaluate the clustering performance, the parameter sensitivity, and the convergence of Algorithm \ref{Algorithm-proposed} on six benchmark datasets.
\subsection{Benchmark Datasets}

We select six datasets of two types: image and text. The key information of the datasets is shown in Table \ref{datasets} and the sample images from two image data sets are illustrated in Figure \ref{samples}. The details of these datasets are given below:
\begin{enumerate}
	\item {\textbf{HW}\footnote{\footnotesize{\texttt{http://archive.ics.uci.edu/ml/datasets/Multiple+\\ 
	Features}}}} contains 2000 images of 0-9 ten-digit classes. Each class has 200 images, which are described by six views. These classes including Profile correlations (216), Fourier coefficients (76), Karhunen coefficients (64), Morphological (6), Pixel averages (240), and Zernike moments (47). The number in brackets represents the dimension of each view. The data we use just includes two views with Profile correlations and Pixel averages.
	\item {\textbf{BBCSprot}\footnote{\footnotesize{\texttt{http://mlg.ucd.ie/datasets/segment.html.}}}} is derived from the BBC Sport section. It contains 544 documents and each document is split into two related segments as views. The dimension of two views are 3183 and 3203 respectively. 
	\item {\textbf{BBC}\footnote{\footnotesize{\texttt{http://mlg.ucd.ie/datasets/segment.html.}}}} is derived from the BBC news corporan. It contains 685 documents and each document is split into four related segments as views, which dimensions are 4659, 4633, 4665 and, 4684 respectively.
	
	\item {\textbf{3Sources}\footnote{\footnotesize{\texttt{http://mlg.ucd.ie/datasets/3sources.html.}}}} is a document dataset collected from BBC, Reuters, and The Guardian. It contains 169 documents and these documents belong to six different themes including technology, health, business, politics, entertainment, and sport.
	
	\item {\textbf{CiteSeer}\footnote{\footnotesize{\texttt{http://lig-membres.imag.fr/grimal/data.html.}}}} is collected from citeseer website. It contains 3312 documents and each document is described by content and citations. 3312 documents can be classified into six classes, including Agents, AI, DB, IR, ML and HCI.

	\item {\textbf{ORL}\footnote{\footnotesize{\texttt{http://www.cl.cam.ac.uk/research/dtg/.}}}} is created by the Olivetti Research Laboratory in Cambridge, England. It is a face dataset containing 400 images of 40 different people. For each subject, images are taken at different times, lights, facial expression (open or closed eyes, smiling or not smiling), and facial details (with glasses or not). Each image uses three kinds of features which called intensity feature, LBP feature, and Gabor feature to obtain three views.
	
\end{enumerate}

\subsection{Compared Method}
We compare our proposed Algorithm \ref{Algorithm-proposed} with the following methods, including 10 state-of-the-art multi-view clustering algorithms. Eight algorithms include four matrix decomposition clustering algorithms, Co-training algorithms and other SOTA multi-view clustering algorithms.
\begin{enumerate}
	\item Perform k-means to every view and get the result of each view, then select the best one as the final result. We call the method \textbf{BKM}.
	\item \textbf{AKM} is regarded as a baseline method. It concatenates all of the views as one view and performs k-means to get the final result.
	\item Kernel-based weighted multi-view clustering (\textbf{MVKKM}) \cite{tzortzis2012kernel} expresses all views by given kernel matrices. A weighted combination of the kernels is learned in parallel to the partitioning.
	\item Multi-view k-means clustering on big data (\textbf{RMKMC}) \cite{cai2013multi} proposes a new robust large-scale multi-view clustering method adaptively to integrate heterogeneous representations of large scale data and then induces structured sparsity-inducing norm to make it more robust to outliers.
	\item A co-training approach for multi-view spectral clustering (\textbf{Co-train}) \cite{kumar_co-training_nodate} has been proposed with a flavor of co-training. They work on the assumption that the true underlying clustering would assign a point to the same cluster irrespective of the view with no hyperparameters. 
	\item Feature extraction via multi-view non-negative matrix factorization with local graph regularization (\textbf{MultiNMF}) \cite{wang2015feature} is motivated by manifold learning and multi-view NMF. The inner-view relatedness between data is taken into consideration. 
	\item Adaptive Structure Concept Factorization for Multi-view Clustering (\textbf{MVCF}) \cite{zhan2018adaptive} is a method for data integration. This method correlates the affinity weights of all views with the inter-view correlation.
	\item Self-weighted multi-view clustering with soft capped norm (\textbf{SCaMVC}) \cite{huang2018self} learns an optimal weight for each view automatically without introducing an additive parameter. It mainly deals with different level noises and outliers by using soft capped norm.
	
	\item Multi-view clustering via deep semi-NMF (\textbf{DMVC}) \cite{zhao_multi-view_nodate} proposes a deep matrix factorization framework for MVC. A graph regularization term is added to a deep NMF framework for preserving the inherent structure of the origin data. It is required that the representation in the last layer of each view is the same.
	\item Auto-weighted multi-view clustering via deep matrix decomposition (\textbf{AwDMVC}) \cite{huang_auto-weighted_2020} learns lower-level hidden attributes for the subsequent clustering task. The weights of different views are automatically assigned without introducing extra hyperparameters.
\end{enumerate}

\begin{table*}[t]
	\caption{The ACC comparison of the different algorithms on six benchmark datasets.}
	\centering
	\label{performance_ACC}
	\setlength{\tabcolsep}{5mm}{
	\begin{tabular}{cccccccc}
		\toprule
		ACC        & BBCSport           		  & 3Source              & BBC                   & CiteSeer             & ORL                  & HW                &Average Rank\\
		\midrule
		BKM   & 42.93                        & 47.97                & 45.37                & 41.43                & 56.00                & 81.95               &6.33\\
		AKM & 47.97                        & 49.77                & 40.36                & 46.02                & 58.25                & 64.90               &6.50\\
		RMKMC      & 45.93                        & 35.39                & 33.82                & 22.77                & 24.50                & 66.10               &8.83\\
		MVKKM      & 40.45                        & 46.39                & 44.92                & 23.75                & 62.50                & 61.90               &8.17\\
		Co-train   & 39.18                        & 33.15                & 32.71                & 26.44                & 72.50                & 80.15               &7.67\\
		MultiNMF   & 57.51                        & 50.28                & 48.26                & 40.22                & 23.75                & 78.54               &5.83\\
		MVCF       & 63.24                        &  58.21   & 65.75                & 47.21                & 66.50                & 76.75               &3.17\\
		ScaMVC     & 43.67                        & 54.23                & 51.95                & 23.47                & 61.75                & 75.20               &6.33\\
		DMVC       & 43.81                        & 44.21                & 49.48                & 24.83                & 77.00                & 38.70               &6.83\\
		AwDMVC     & 70.76                        & 55.86                & 62.34                & 48.25                & -                    & -                   &5.33\\
		Ours    &  \textbf{91.73} 		   & \textbf{70.41}                &  \textbf{71.68}   &  \textbf{53.86}   & \textbf{79.25}             & \textbf{84.75}  &\textbf{1.00}\\
		\bottomrule 
	\end{tabular}}
\end{table*}

\begin{table*}[t]
	\caption{The NMI comparison of the different algorithms on six benchmark datasets.}
	\centering
	\label{performance_NMI}
	\setlength{\tabcolsep}{5mm}{
	\begin{tabular}{cccccccc}
		\toprule
		NMI & BBCSport                     & 3Source                      & BBC                          & CiteSeer                     & ORL                          & HW                           &Average Rank\\
		\midrule
		BKM     & 20.75                        & 25.73                        & 27.56                        & 16.41                        & 74.44                        & 75.89                        &6.33\\
		AKM   & 27.64                        & 30.58                        & 22.06                        & 20.21                        & 77.22                        & 62.23                        &6.33\\
		RMKMC        & 24.27                        & 15.27                        & 12.11                        & 11.43                        & 56.89                        & 76.57                        &8.17\\
		MVKKM        & 19.09                        & 26.71                        & 20.96                        & 11.85                        & 77.97                        & 65.64                        &8.17\\
		Co-train     & 16.48         & 10.41                        &10.94       & 12.25                        & 86.61                        & \textbf{ 76.59}    &7.67\\
		MultiNMF     & 37.96                        & 42.47                        & 27.37                        & 20.10                        & 37.98                        & 74.64                        &5.50\\
		MVCF         & 40.45                        & 48.16                        & 42.80                        & 21.10                        & 83.90                        & 68.74                         &3.67\\
		ScaMVC       & 20.36                        & 31.57                        & 20.18                        & 12.29                        & 78.92                        & 75.64                        &6.67\\
		DMVC         & 26.04                        & 33.35                        & 20.16                        & 13.01                        & 88.00                        & 38.65                        &6.50\\
		AwDMVC       & 46.82                        & 46.09 &42.82    & 22.01    & -                            & -                            &5.17\\
		Ours      & \textbf{79.38}    & \textbf{59.53}    &\textbf{45.60} &\textbf{25.38} & \textbf{90.75}    & 74.11 &\textbf{1.83}\\
		\bottomrule                       
	\end{tabular}}
\end{table*}

\begin{table*}[ht]
	\caption{The PUR comparison of the different algorithms on six benchmark datasets.}
	\centering
	\label{performance_PUR} 
	\setlength{\tabcolsep}{5mm}{
	\begin{tabular}{cccccccc}
		\toprule
		PUR & BBCSport                     & 3Source                      & BBC                          & CiteSeer                     & ORL                          & HW                           &Average Rank\\
		\midrule
		BKM     & 44.84                        & 51.47                        & 46.47                        & 42.59                        & 62.00                        & 81.95                        &6.33\\
		AKM   & 49.36                        & 56.32                        & 40.63                        & 45.70                        & 63.00                        & 68.30                        &6.83\\
		RMKMC        & 45.94                        & 36.64                        & 33.82                        & 21.76                        & 24.50                        & 70.05                        &9.00\\
		MVKKM        & 37.61                        & 49.88                        & 46.35                        & 24.53                        & 68.50                        & 65.50                        &8.50\\
		Co-train     & 43.68                        & 34.92                        & 33.15                        & 28.64                        & 76.68                        & 80.92    &7.50\\
		MultiNMF     & 59.23                        & 63.05                        & 48.25                        & 41.92                        & 23.75                        & 79.81                        &5.33\\
		MVCF         & 63.42                        & 60.05                        & 65.84                        & 48.77                        & 70.25                        & 76.80                        &3.83\\
		ScaMVC       & 44.26                        & 60.33                        & 52.56                        & 23.96                        & 66.00                        & 75.20                        &6.67\\
		DMVC         & 51.36                        & 62.80 & 48.38                        & 28.14                        & 79.75 & 38.60                        &5.50\\
		AwDMVC       & 65.99                        & 62.45 & 63.80    & 50.00    & -                            & -                            &5.50\\
		Ours      & \textbf{ 91.73}    & \textbf{ 80.47}    &\textbf{72.12} &\textbf{55.98} & \textbf{82.25}    &\textbf{ 84.90}  &\textbf{1.00}\\
		\bottomrule
	\end{tabular}}
\end{table*}

\subsection{Experiment Setup}

For the proposed mehtod, all original feature matrices should be normalized firstly. We set the number of clusters is the true number of classes for each dataset. The trade-off parameters $\beta$ is selected from $\left[2^{-7},2^{-5},\dots,2^{5},2^{7}\right]$. We assume that the layer size should be correlated with the number of clusters, so we design two schemes with one layer size $p_2=\left[l_1,l_2\right]$ and another layer size $p_3=\left[l_1,l_2,k\right]$. Where $l_1,l_2$ in $p_2$ are chosen from $\left[4k,8k,12k\right]$ and $\left[k,2k,3k\right]$ respectively and $l_1,l_2$ in $p_3$ are chosen from $\left[7k,11k,15k \right]$ and $\left[2k,3k,4k \right]$ respectively. The reason why the third layer of $p_3$ is fixed to $k$ will be explained in the subsection \ref{last_layer}. For these compared methods, we obtain their paper and code from the autors' websites and obey the setting of the hyper-parameters in the paper. 

The clustering performance is evaluated by three widely used criteria, including clustering accuracy (ACC), normalized mutual information (NMI), and purity (PUR). We repeat each experiment 50 times to avoid the effect of the random initialization and save the best result. All experiments are conducted on a desktop computer with Intel i9-9900K CPU @ 3.60GHz$\times$16 and 64GB RAM, MATLAB 2017a (64bit).

\subsection{Experiment Results}

Table \ref{performance_ACC}, Table \ref{performance_NMI} and Table \ref{performance_PUR} show the clustering performance which is measured by ACC, NMI, and PUR on six benchmark datasets. The best results of all datasets in all algorithms mark in bold. Based on these tables, we can obtain the following conclusions:
\begin{itemize} 
	\item As for Table \ref{performance_ACC} which measured by ACC, 
	we can find it seems better to connect all features then do k-means than perform k-means on all single views then get the best in most of the time. So using all information of data is always better than a certain aspect. It is prominent that the performance of our proposed method is always best. The ACC of our algorithm exceeds the second best method by $20.97\%$, $12.20\%$, $5.93\%$, $5.61\%$, $2.25\%$, and $2.80\%$ on BBCSport, 3Sources, BBC, CiteSeer, ORL and HW respectively. What stands out is the performance of on the image datasets ORL and HW. Both of them just have little increase in ACC and the algorithm of Co-train exceeds our algorithm by $2.21\%$ on NMI reported in Table \ref{performance_NMI}. It proves that our algorithm is more suitable for text datasets from another angle. 
	\item Comparing with the DMVC, our proposed algorithm has achieved good results on six benchmark datasets and improved clustering performance. And both of them use a deep semi-NMF framework. The results show that real-time reconstruction of the global graph instead of a fixed graph structure will learn a better representation for the original data and a global consensus graph.
	\item AwDMVC also uses the framework of deep NMF. It assigns a weight to each view automatically when learning feature representation layer by layer without other hyperparameters being introduced. We outperform AwDMVC on all datasets by a large margin, which shows the merits of combining the representation learning and consensus graph reconstructing.
\end{itemize}

In addition, we report NMI in Table \ref{performance_NMI} and PUR in Table \ref{performance_PUR} and get the same conclusions as Table \ref{performance_ACC}.

As a result, we have demonstrated our proposed method is effective compared with other state-of-the-art methods by analyzing the above experimental results. We attribute the superiority of our proposed algorithm with two factors: i\text{)} Our proposed method is based on a deep matrix decomposition framework, so it is can more likely find the meaningful representation layer by layer. ii\text{)} We abandon the original graph structure which results in a bad clustering effect and use the learned new good representation to reconstruct a consistent graph for clustering. iii\text{)} We propose a framework that unifies representation learning and consensus graph constructing, so learning representation and reconstructing the graph can mutually promote each other.

\begin{figure*}
	\centering
	\begin{minipage}{0.24\linewidth}
		\vspace{3pt}
		\centerline{\includegraphics[width=\textwidth]{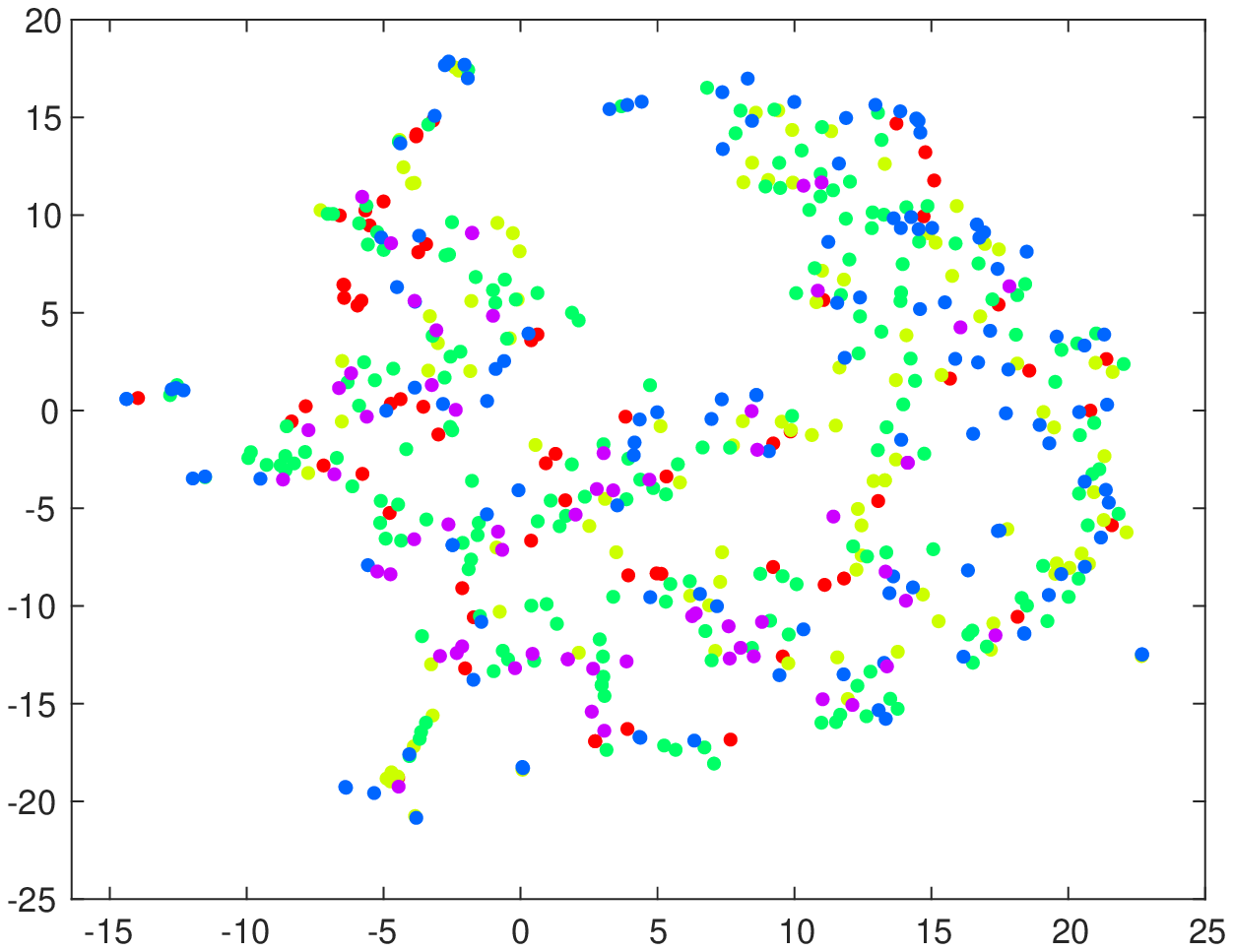}}
		\centerline{(a) BBCSports(iter=1)}
	\end{minipage}
	\begin{minipage}{0.24\linewidth}
		\vspace{3pt}
		\centerline{\includegraphics[width=\textwidth]{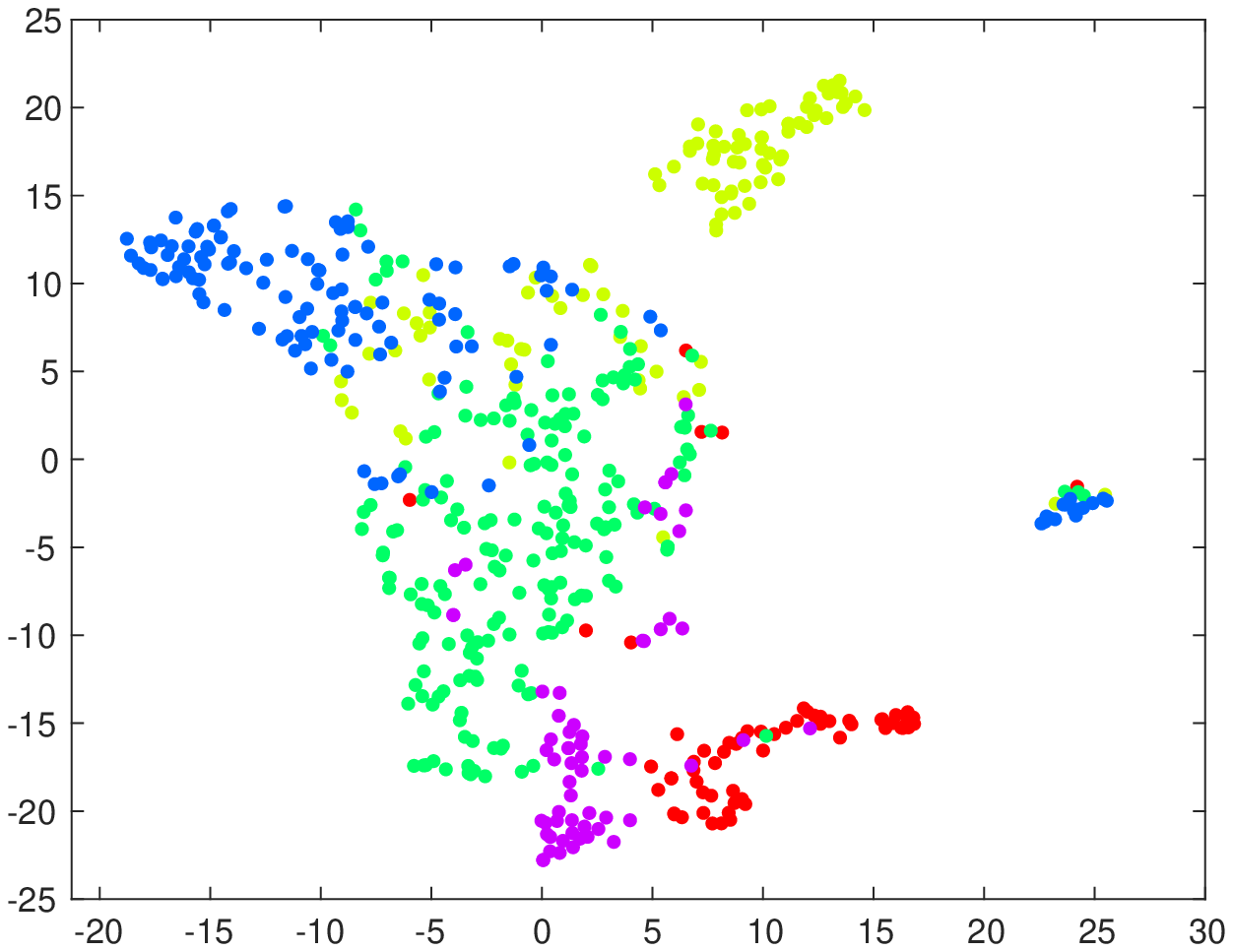}}
		\centerline{(b) BBCSports(iter=10)}
	\end{minipage}
	\begin{minipage}{0.24\linewidth}
		\vspace{3pt}
		\centerline{\includegraphics[width=\textwidth]{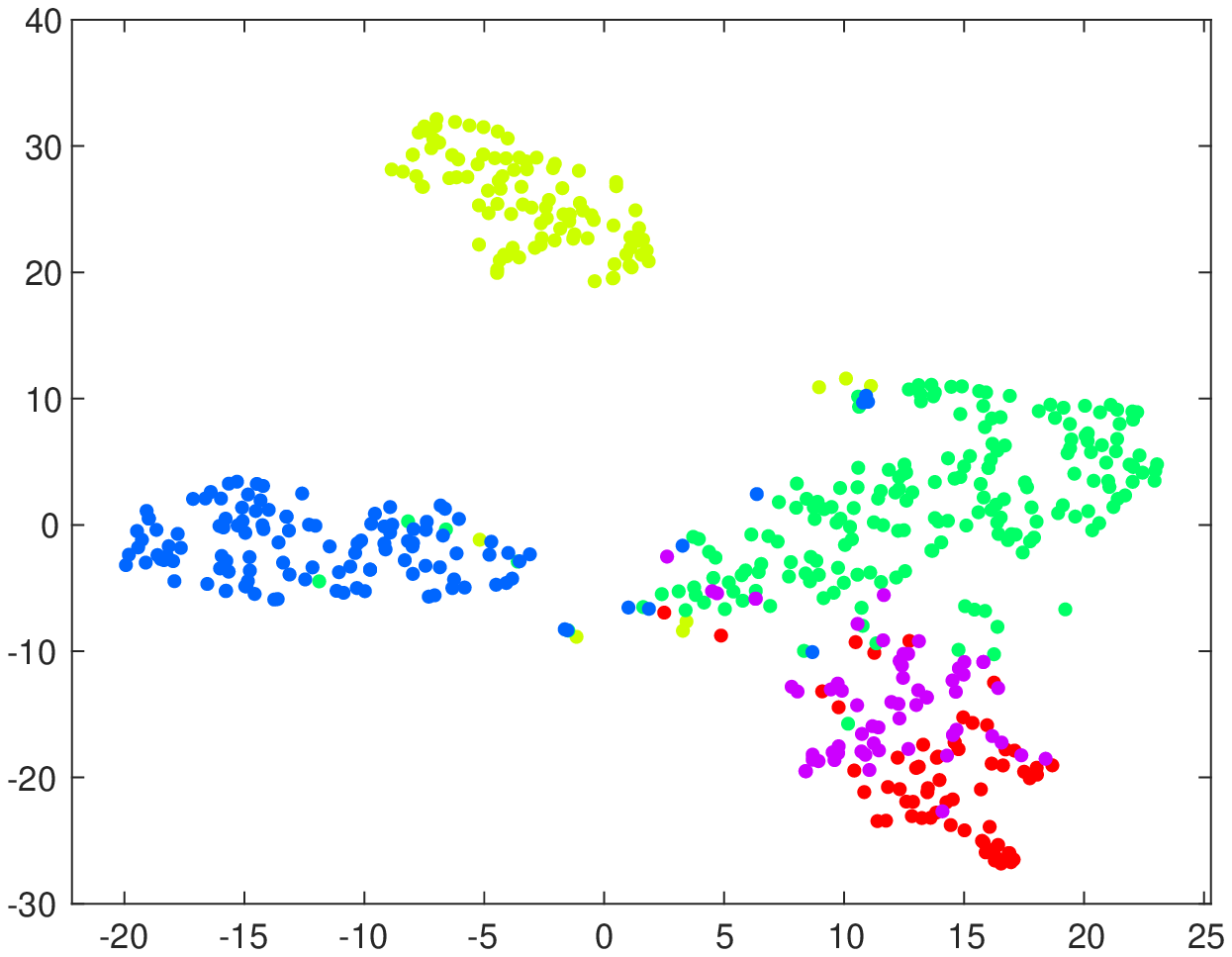}}
		\centerline{(c) BBCSports(iter=50)}
	\end{minipage}
	\begin{minipage}{0.24\linewidth}
		\vspace{3pt}
		\centerline{\includegraphics[width=\textwidth]{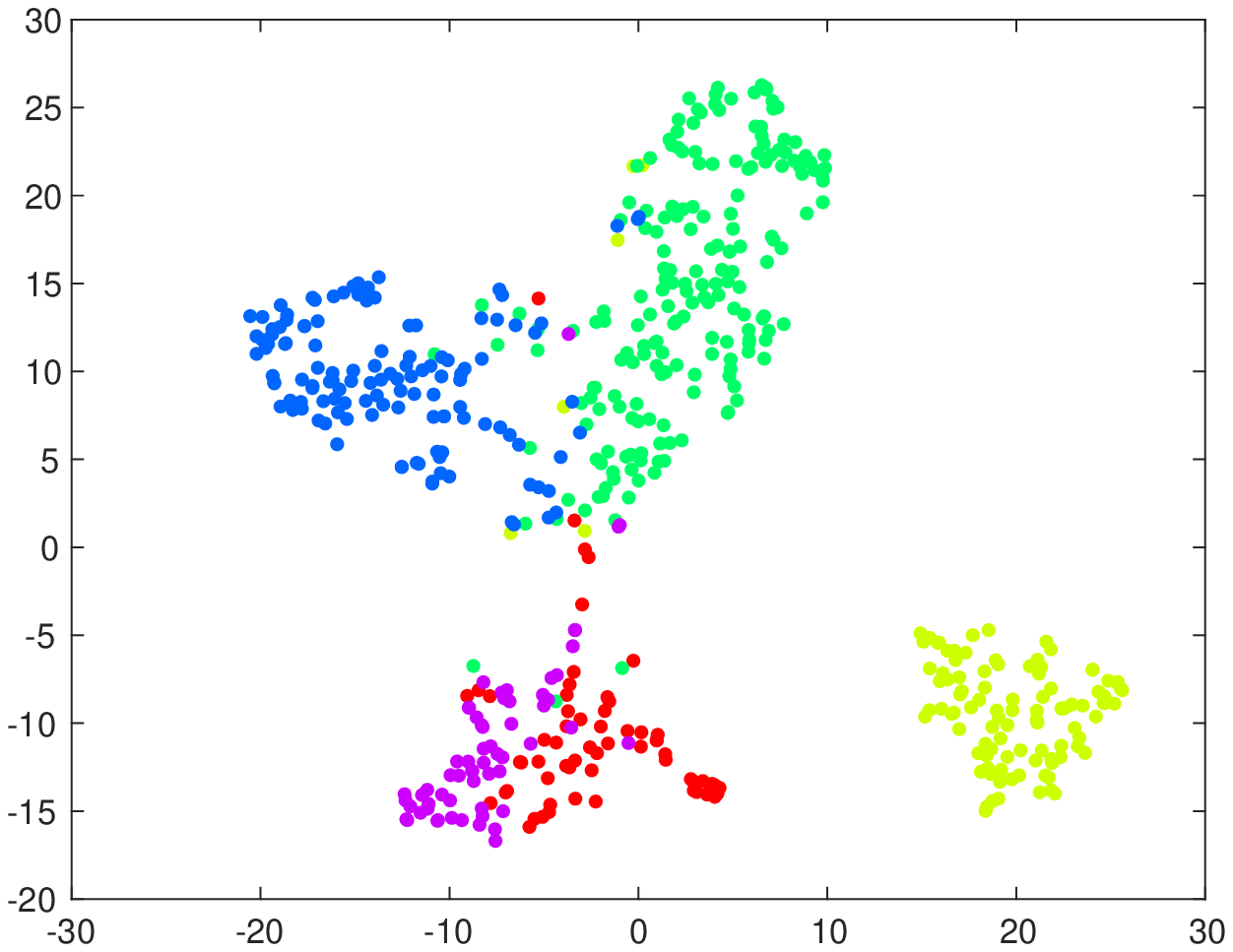}}
		\centerline{(d) BBCSports(iter=150)}
	\end{minipage}

	\begin{minipage}{0.24\linewidth}
		\vspace{3pt}
		\centerline{\includegraphics[width=\textwidth]{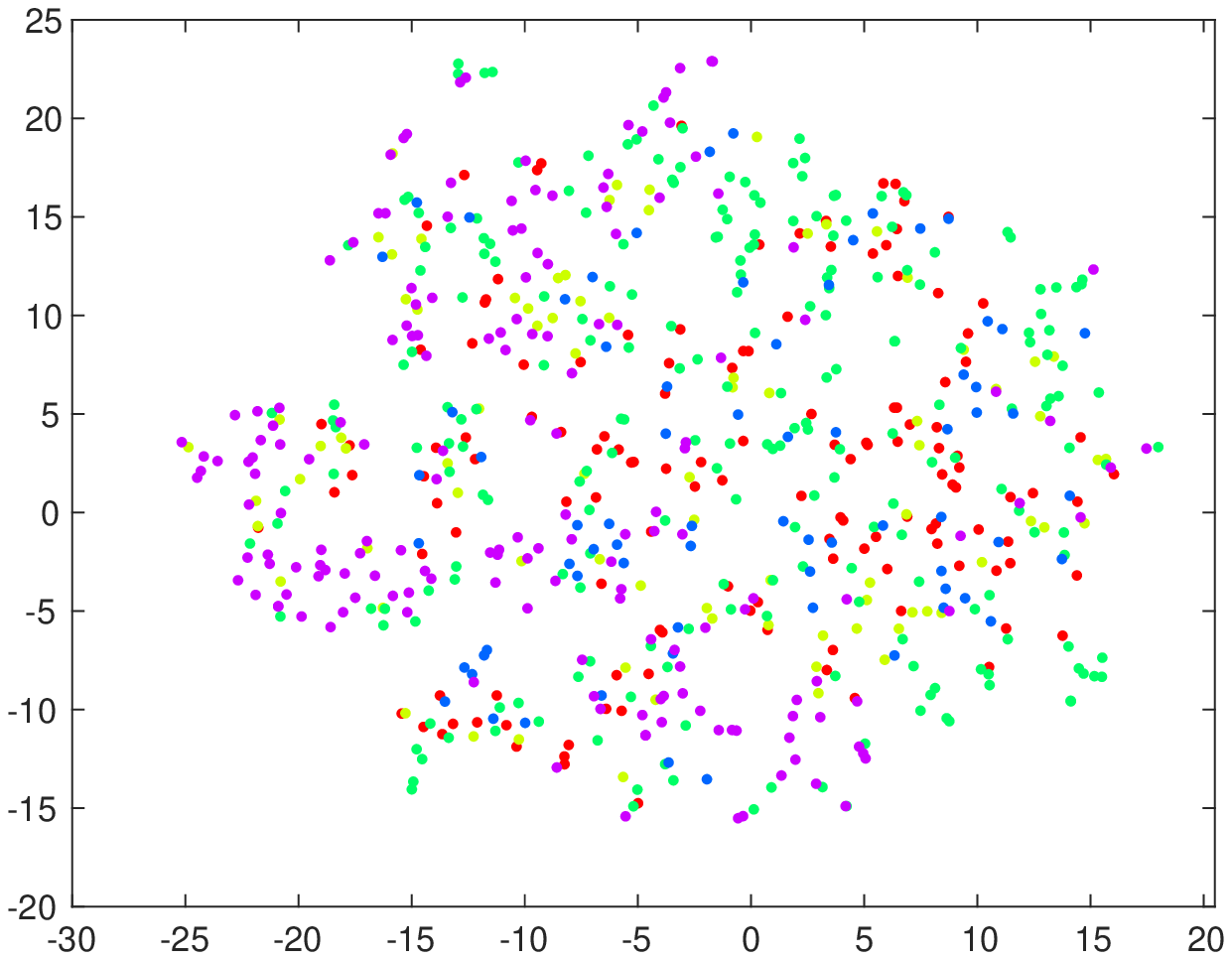}}
		\centerline{(e) BBC(iter=1)}
	\end{minipage}
	\begin{minipage}{0.24\linewidth}
		\vspace{3pt}
		\centerline{\includegraphics[width=\textwidth]{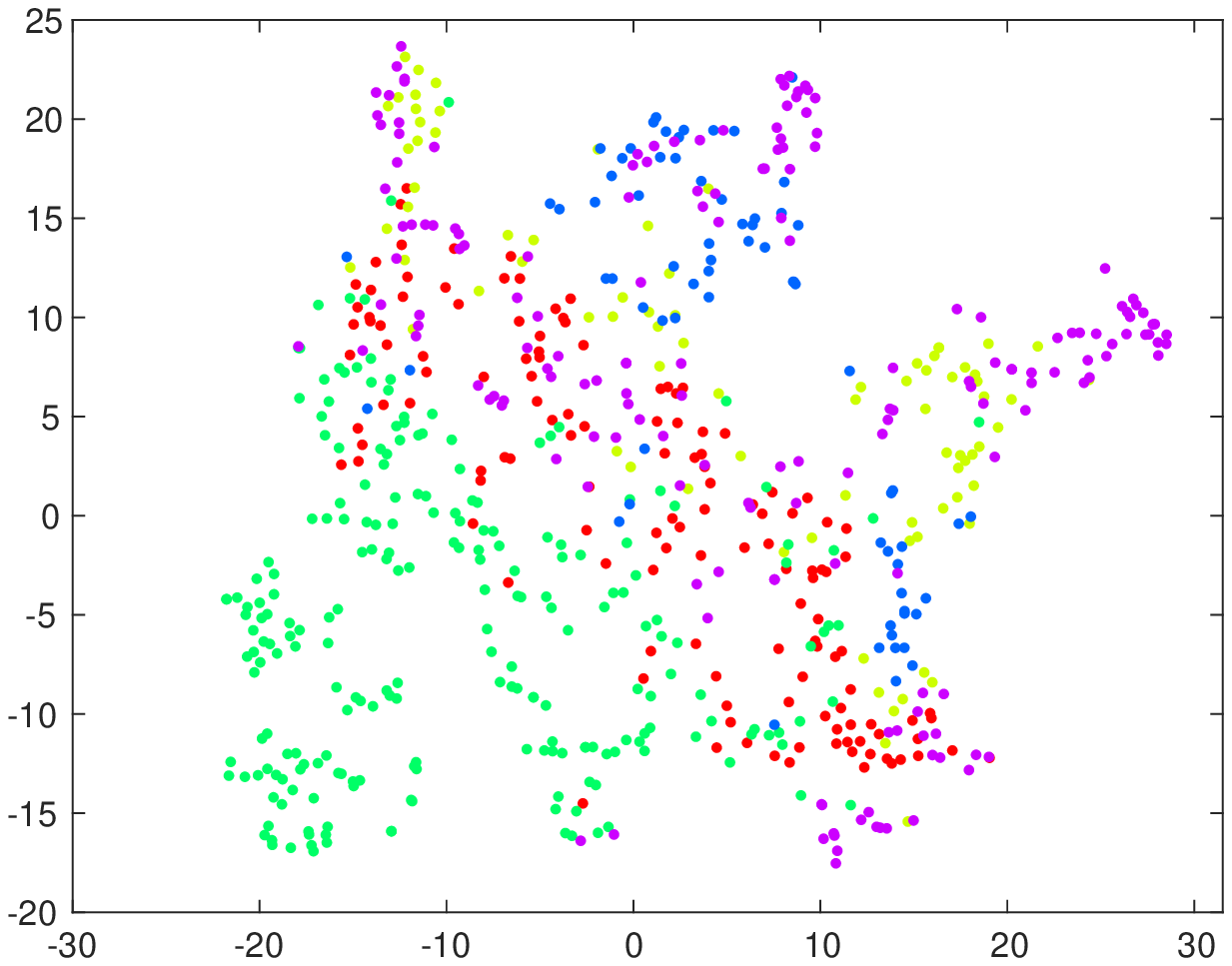}}
		\centerline{(f) BBC(iter=10)}
	\end{minipage}
	\begin{minipage}{0.24\linewidth}
		\vspace{3pt}
		\centerline{\includegraphics[width=\textwidth]{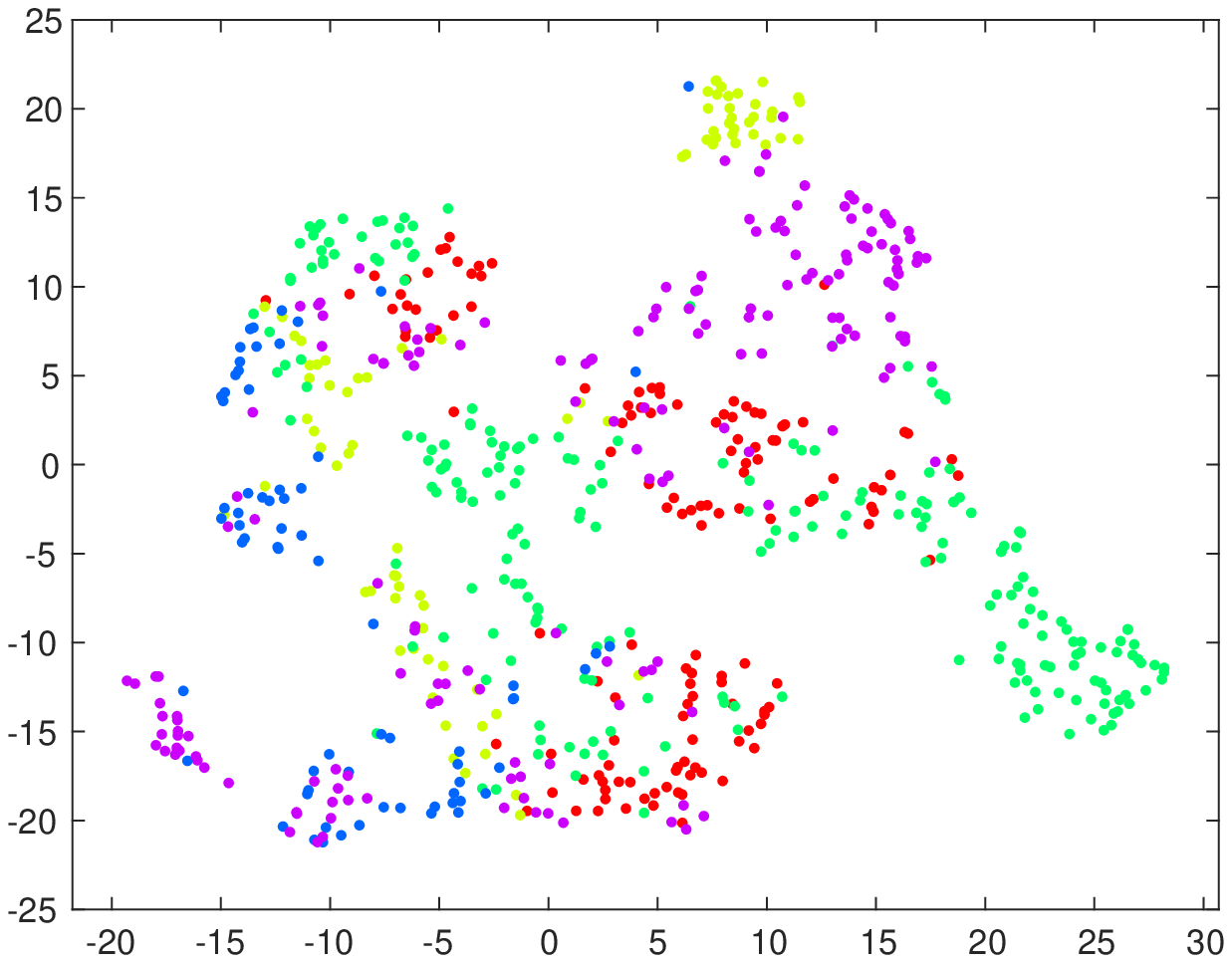}}
		\centerline{(g) BBC(iter=50)}
	\end{minipage}
	\begin{minipage}{0.24\linewidth}
		\vspace{3pt}
		\centerline{\includegraphics[width=\textwidth]{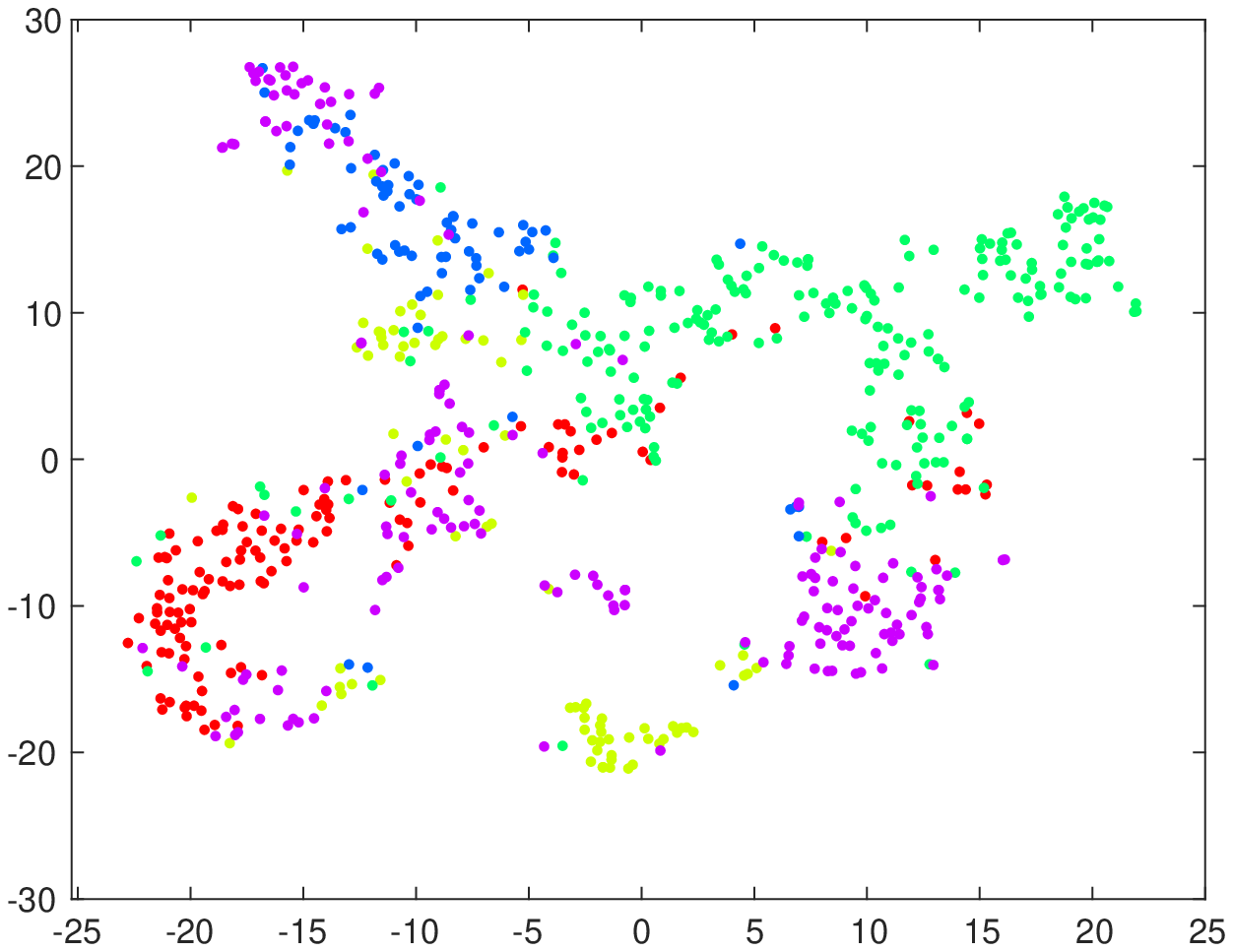}}
		\centerline{(h) BBC(iter=150)}
	\end{minipage}

	\begin{minipage}{0.24\linewidth}
		\vspace{3pt}
		\centerline{\includegraphics[width=\textwidth]{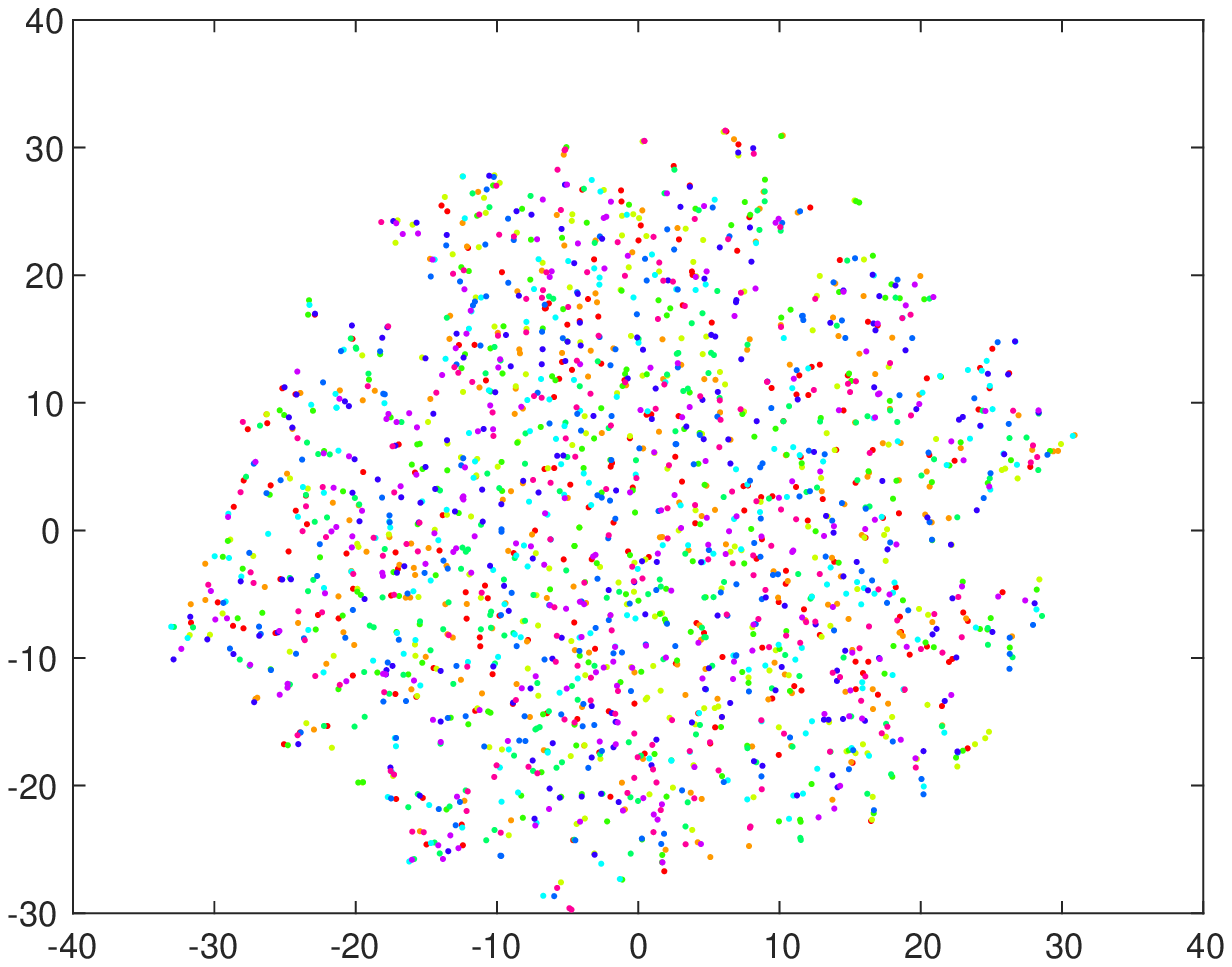}}
		\centerline{(i) HW(iter=1)}
	\end{minipage}
	\begin{minipage}{0.24\linewidth}
		\vspace{3pt}
		\centerline{\includegraphics[width=\textwidth]{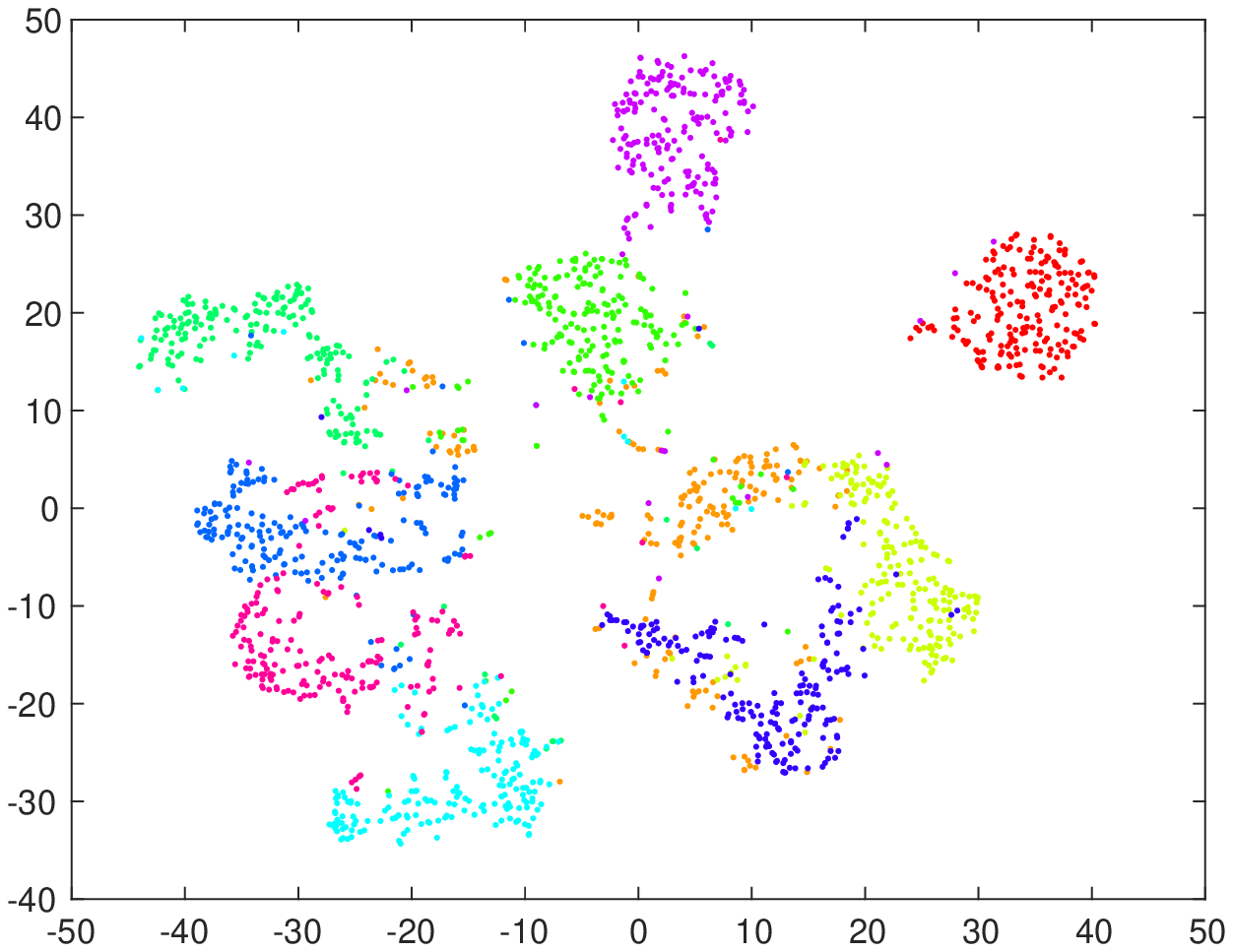}}
		\centerline{(j) HW(iter=10)}
	\end{minipage}
	\begin{minipage}{0.24\linewidth}
		\vspace{3pt}
		\centerline{\includegraphics[width=\textwidth]{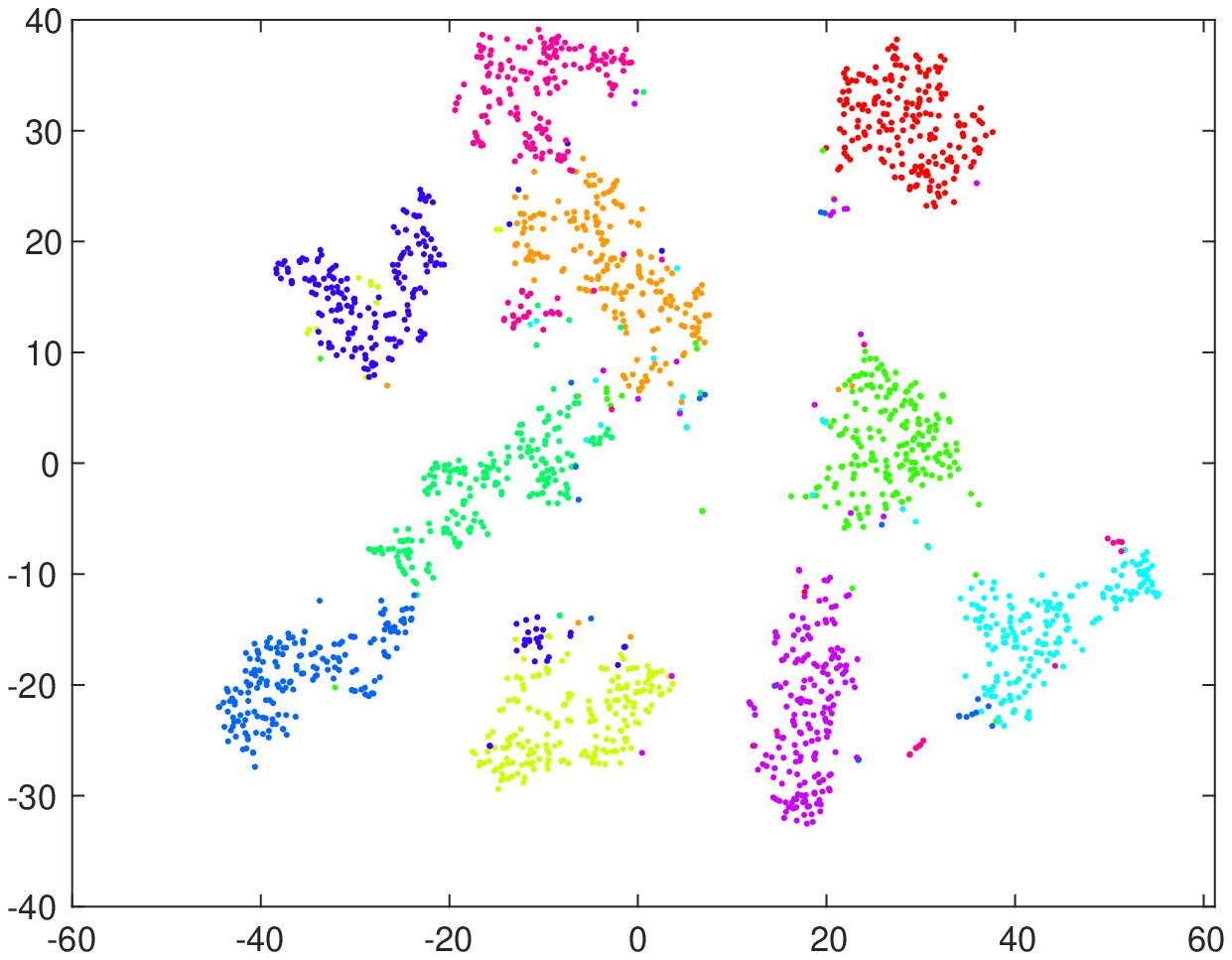}}
		\centerline{(k) HW(iter=50)}
	\end{minipage}
	\begin{minipage}{0.24\linewidth}
		\vspace{3pt}
		\centerline{\includegraphics[width=\textwidth]{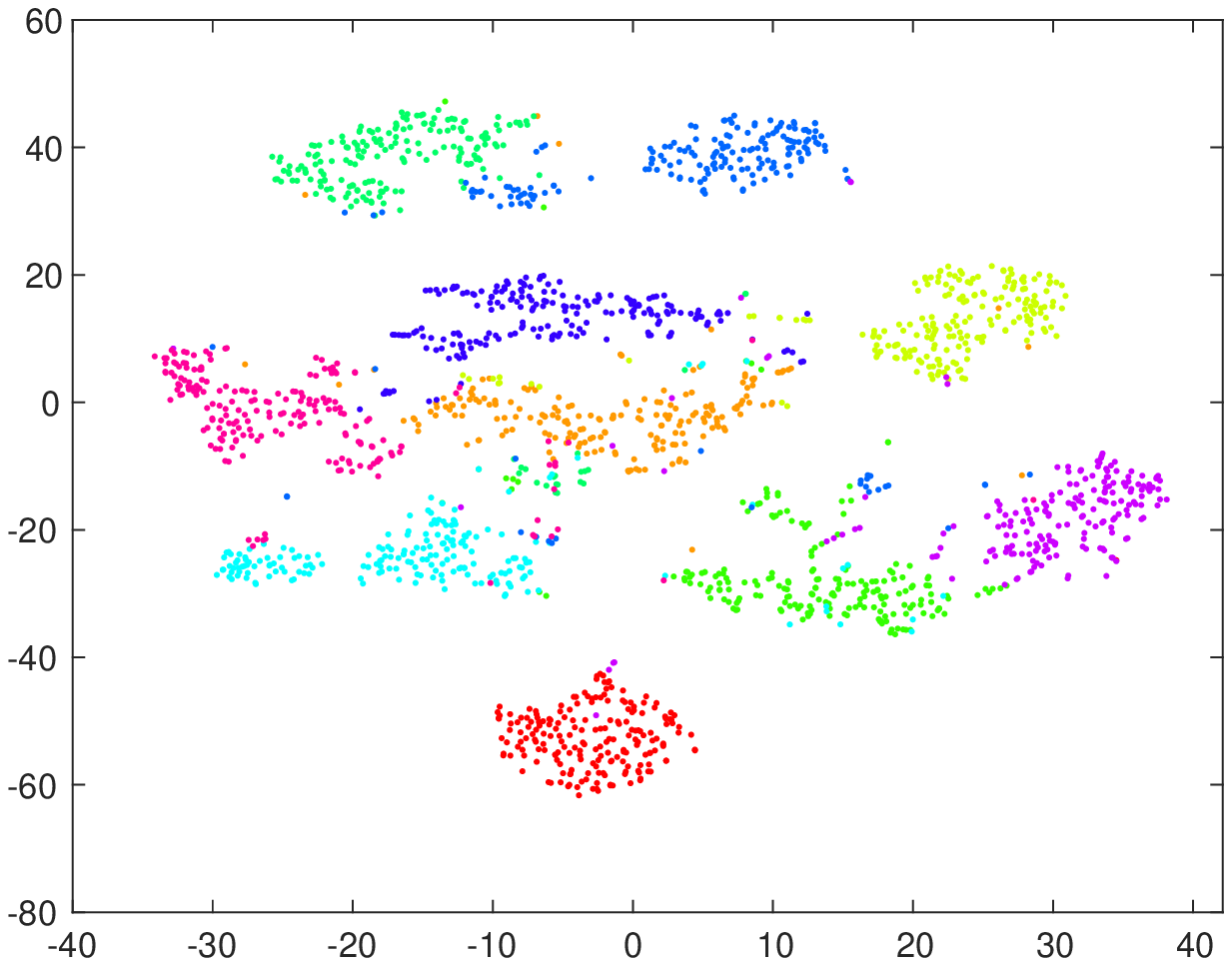}}
		\centerline{(l) HW(iter=150)}
	\end{minipage}
	
	\caption{The proposed algorithm use t-SNE \cite{van2008visualizing} on BBCSports, BBC and HW when iterations are 1, 10, 50 and 150, The different colors indicate different classes for each dataset.}
	\label{VisualizationDifferentIteration}
\end{figure*} 

\begin{figure*}
	
	\begin{minipage}{0.32\linewidth}
		\vspace{3pt}
		\centerline{\includegraphics[width=\textwidth]{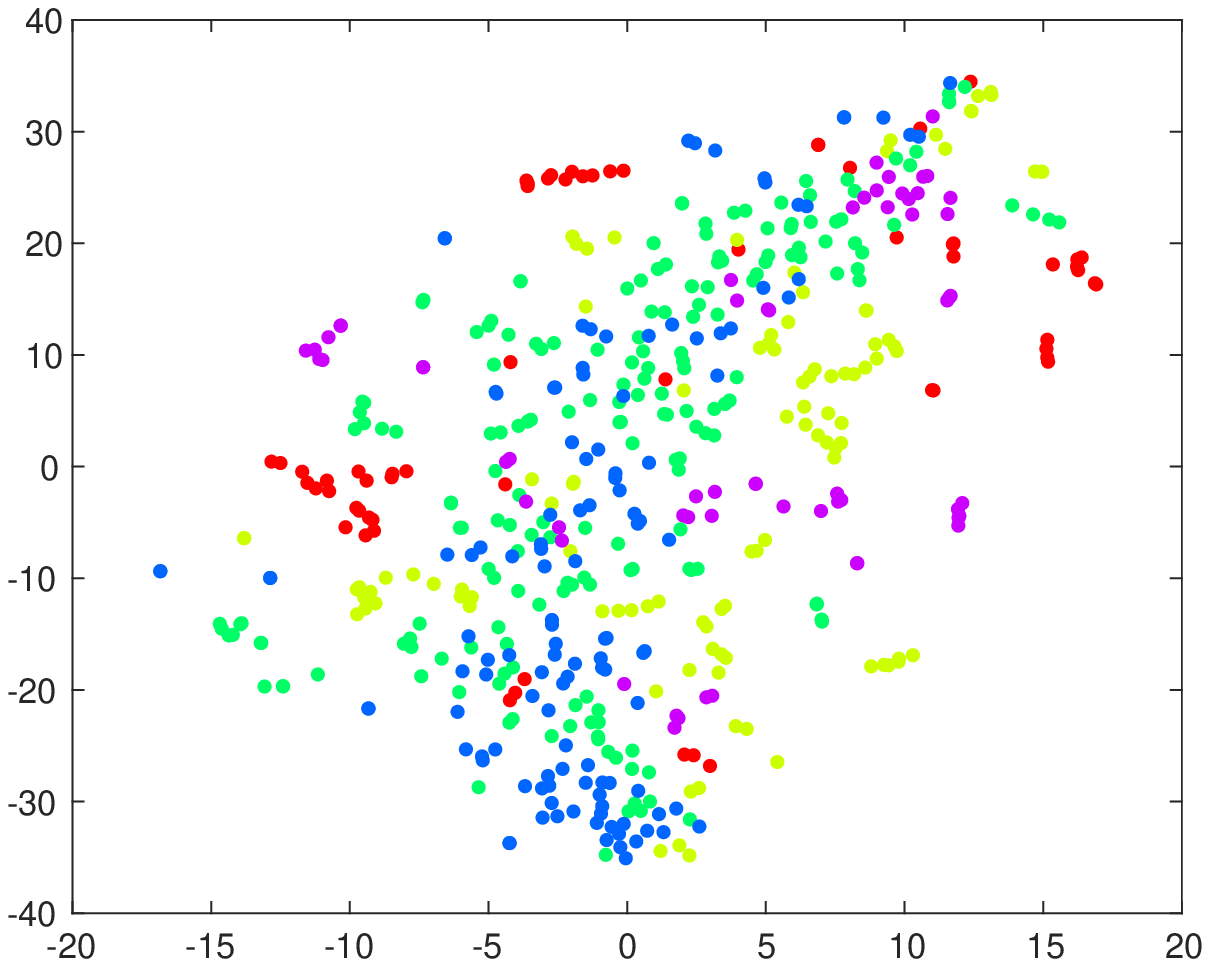}}
		\centerline{(a) DMVC(BBCSports)}
	\end{minipage}
	\begin{minipage}{0.32\linewidth}
		\vspace{3pt}
		\centerline{\includegraphics[width=\textwidth]{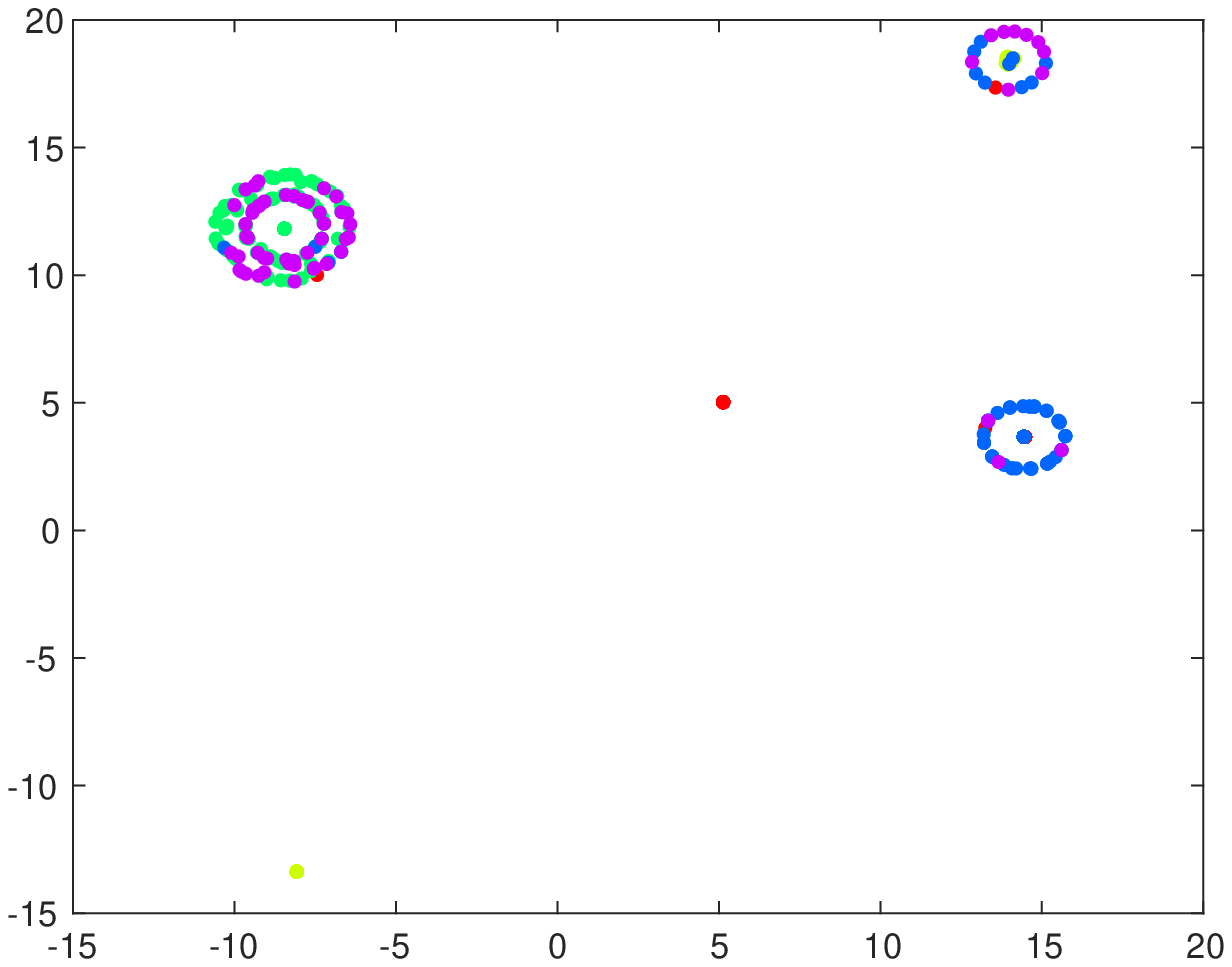}}
		\centerline{(b) AwDMVC(BBCSports)}
	\end{minipage}
	\begin{minipage}{0.32\linewidth}
		\vspace{3pt}
		\centerline{\includegraphics[width=\textwidth]{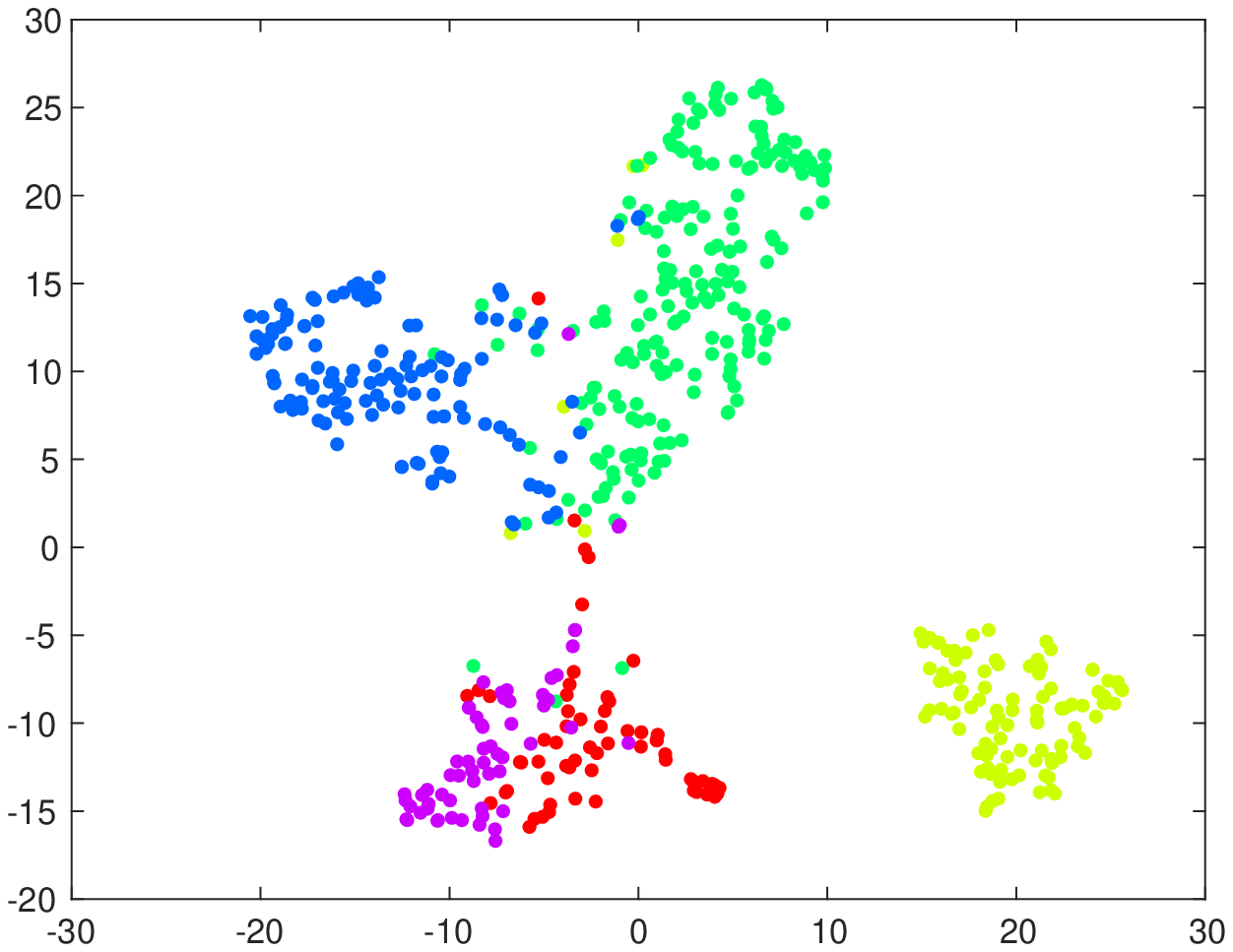}}
		\centerline{(c) Ours(BBCSports)}
	\end{minipage}

	 \begin{minipage}{0.32\linewidth}
		\vspace{3pt}
		\centerline{\includegraphics[width=\textwidth]{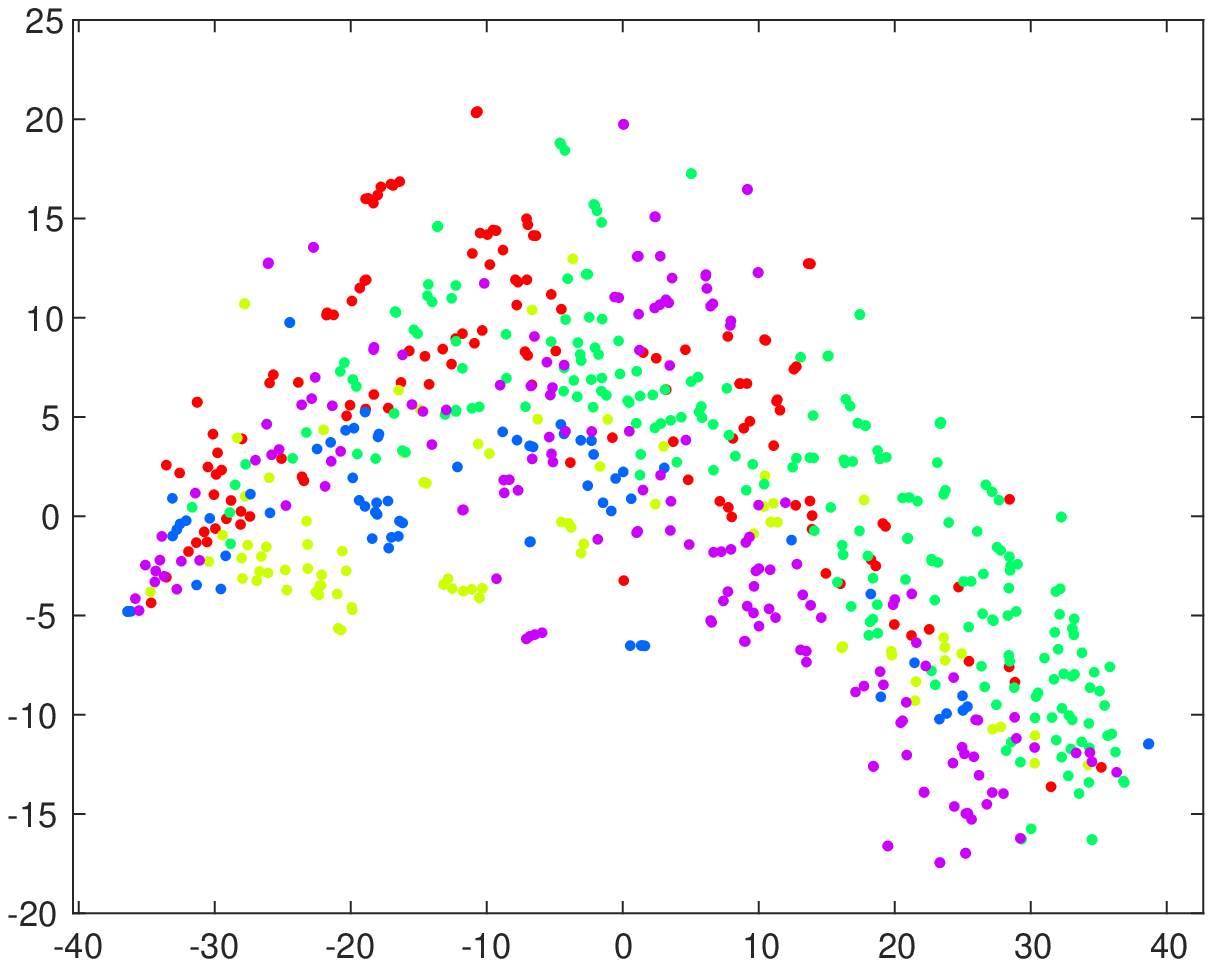}}
		\centerline{(d) DMVC(BBC)}
	\end{minipage}
	\begin{minipage}{0.32\linewidth}
		\vspace{3pt}
		\centerline{\includegraphics[width=\textwidth]{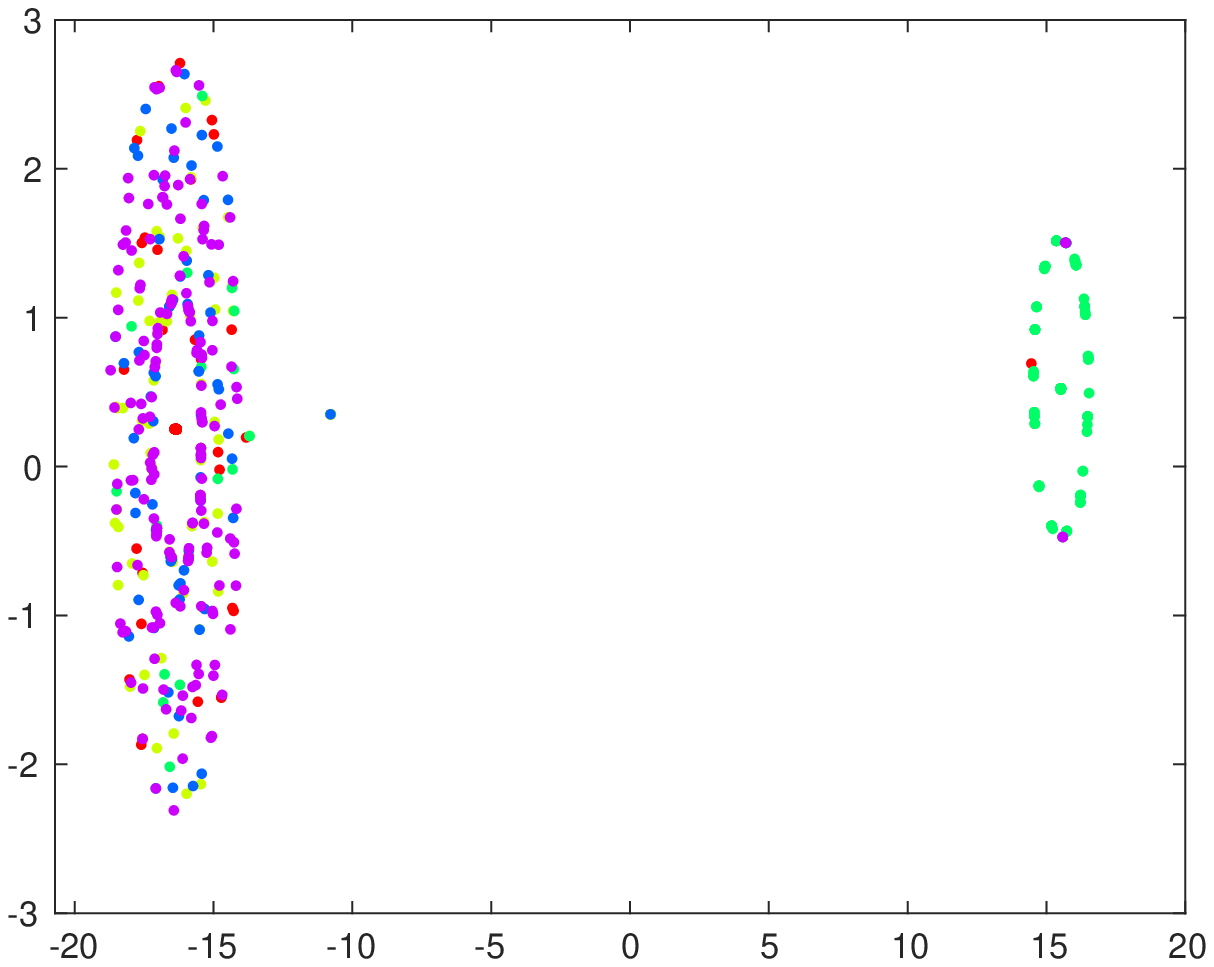}}
		\centerline{(e) AwDMVC(BBC)}
	\end{minipage}
	\begin{minipage}{0.32\linewidth}
		\vspace{3pt}
		\centerline{\includegraphics[width=\textwidth]{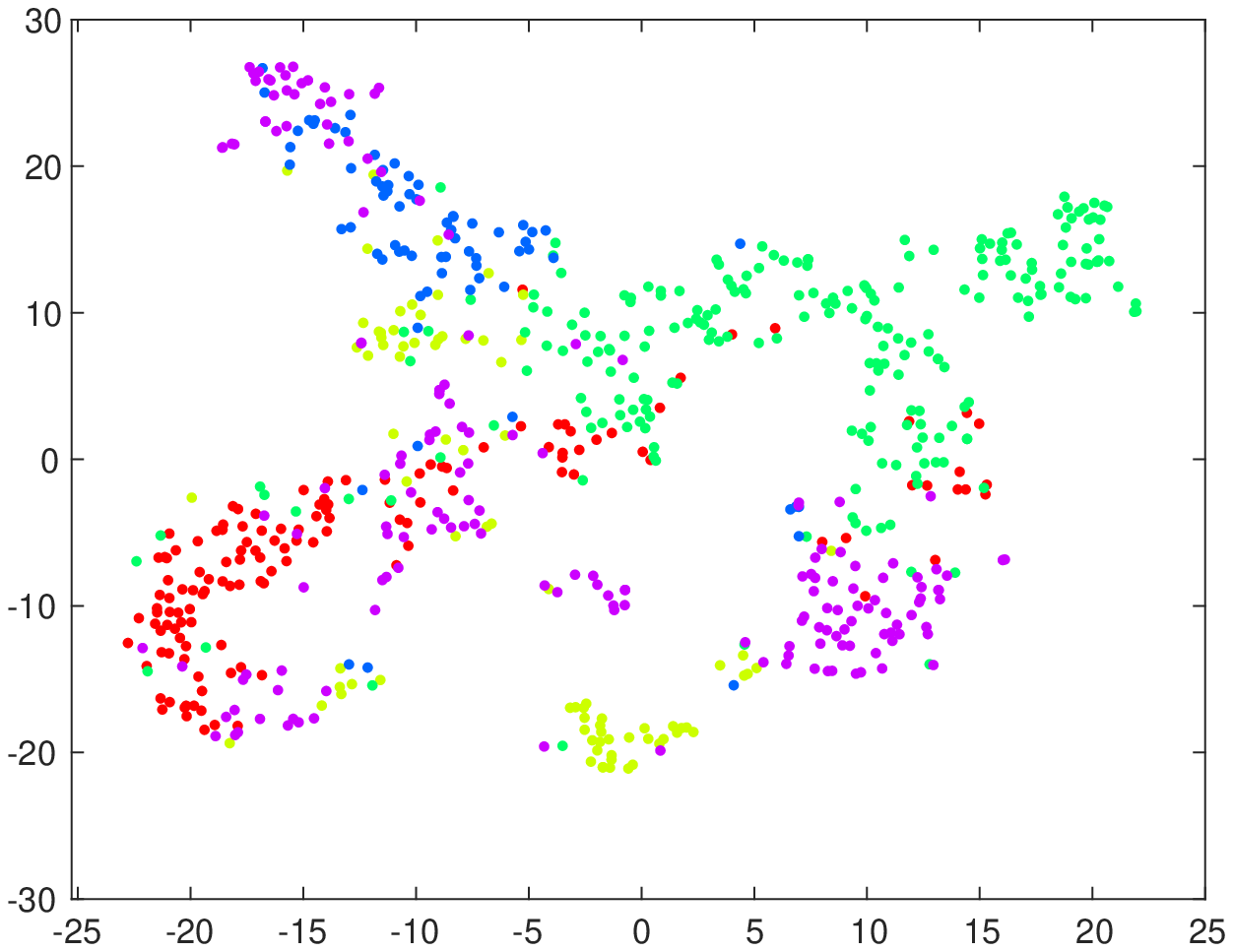}}
		\centerline{(f) Ours(BBC)}
	\end{minipage}

	\caption{The proposed algorithm, DMVC and AwDMVC use t-SNE on BBCSports and BBCwhen, The different colors indicate different classes for each dataset.}
	\label{VisualizationDifferentMethod}
\end{figure*}

\begin{figure*}
	\begin{minipage}{0.32\linewidth}
		\vspace{3pt}
		\centerline{\includegraphics[width=\textwidth]{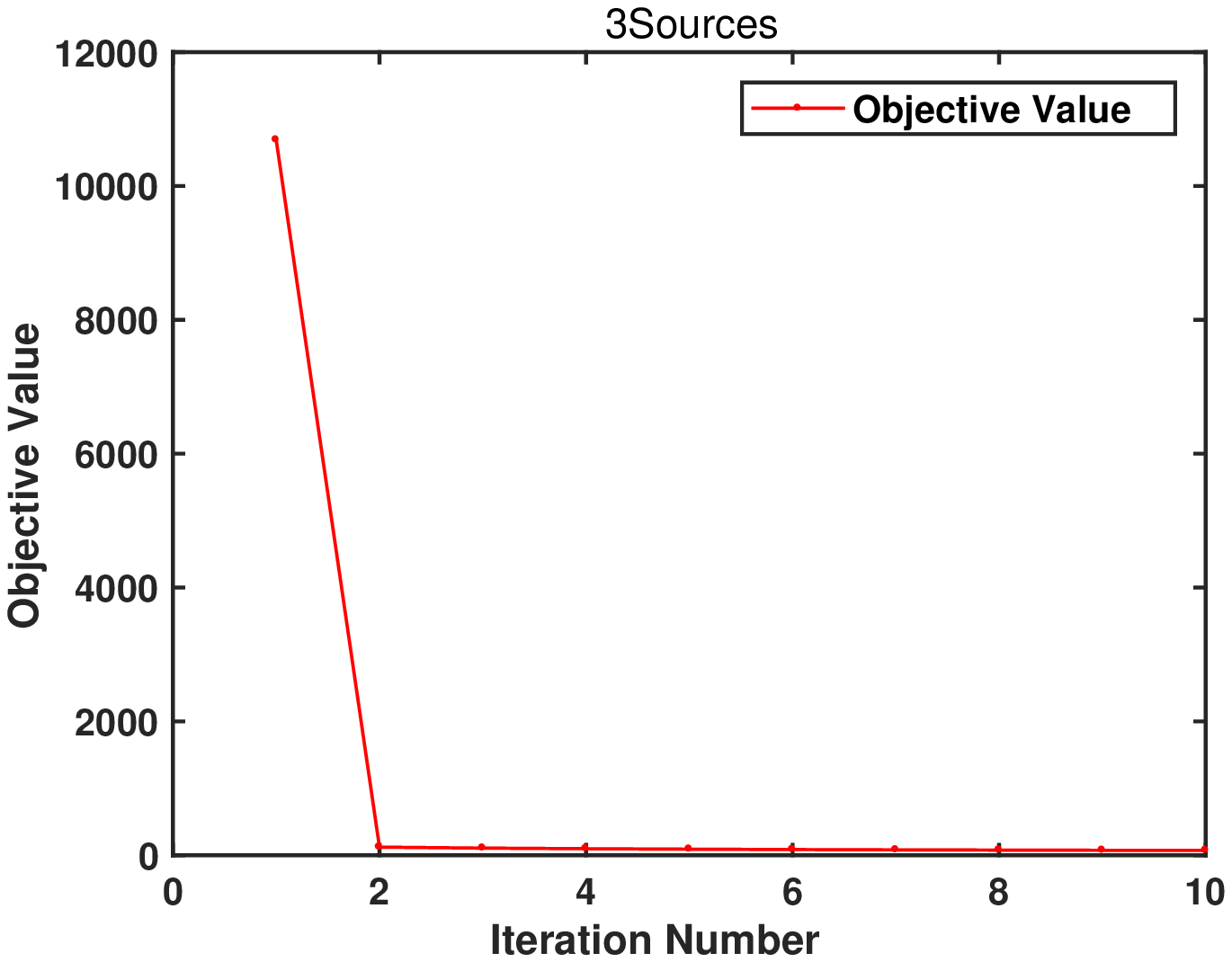}}
		\centerline{(a) 3Sources}
	\end{minipage}
	\begin{minipage}{0.32\linewidth}
		\vspace{3pt}
		\centerline{\includegraphics[width=\textwidth]{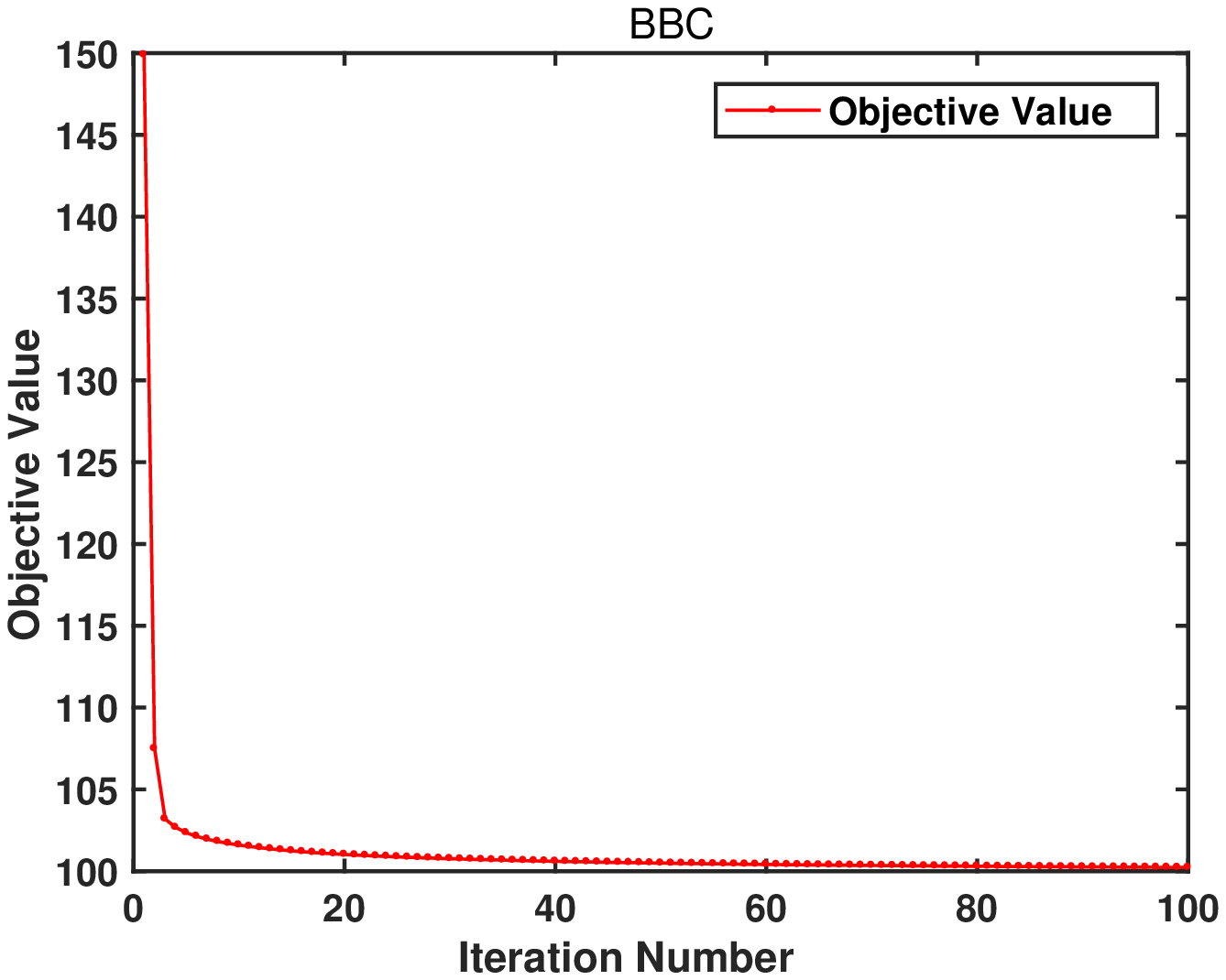}}
		\centerline{(b) BBC}
	\end{minipage}
	\begin{minipage}{0.32\linewidth}
		\vspace{3pt}
		\centerline{\includegraphics[width=\textwidth]{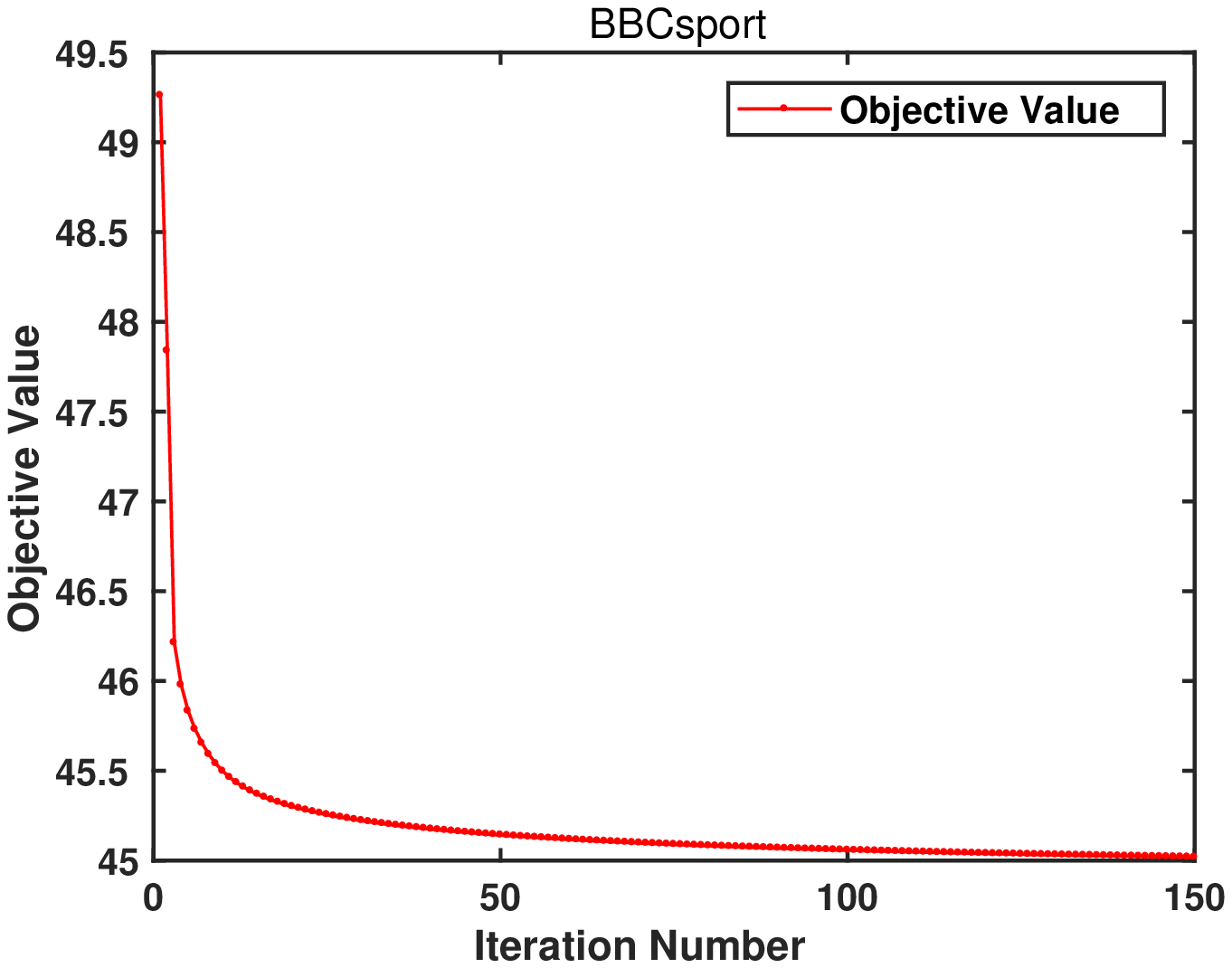}}
		\centerline{(c) BBCSport}
	\end{minipage}
	
	\begin{minipage}{0.32\linewidth}
	\vspace{3pt}
		\centerline{\includegraphics[width=\textwidth]{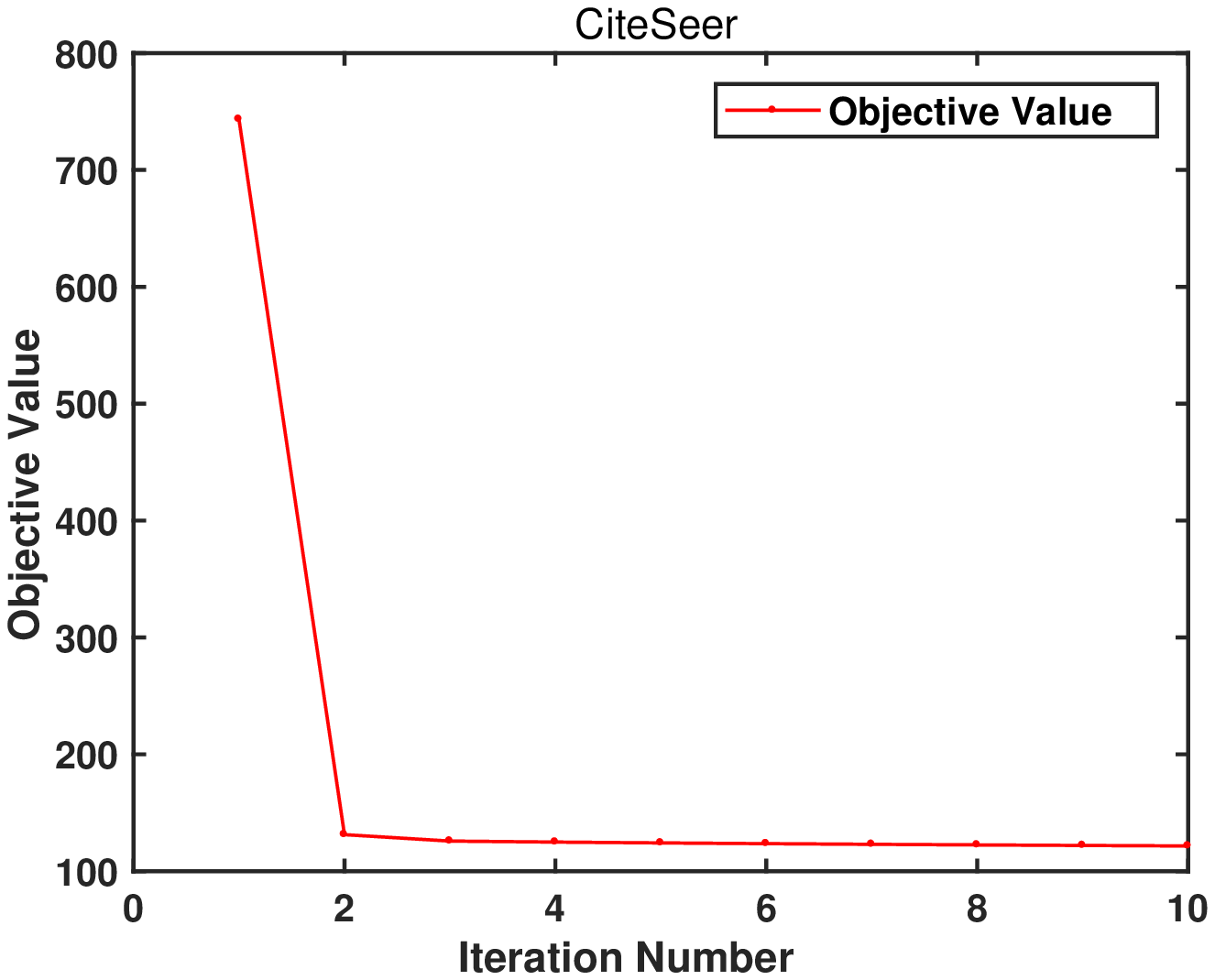}}
		\centerline{(d) CiteSeer}
	\end{minipage}
	\begin{minipage}{0.32\linewidth}
		\vspace{3pt}
		\centerline{\includegraphics[width=\textwidth]{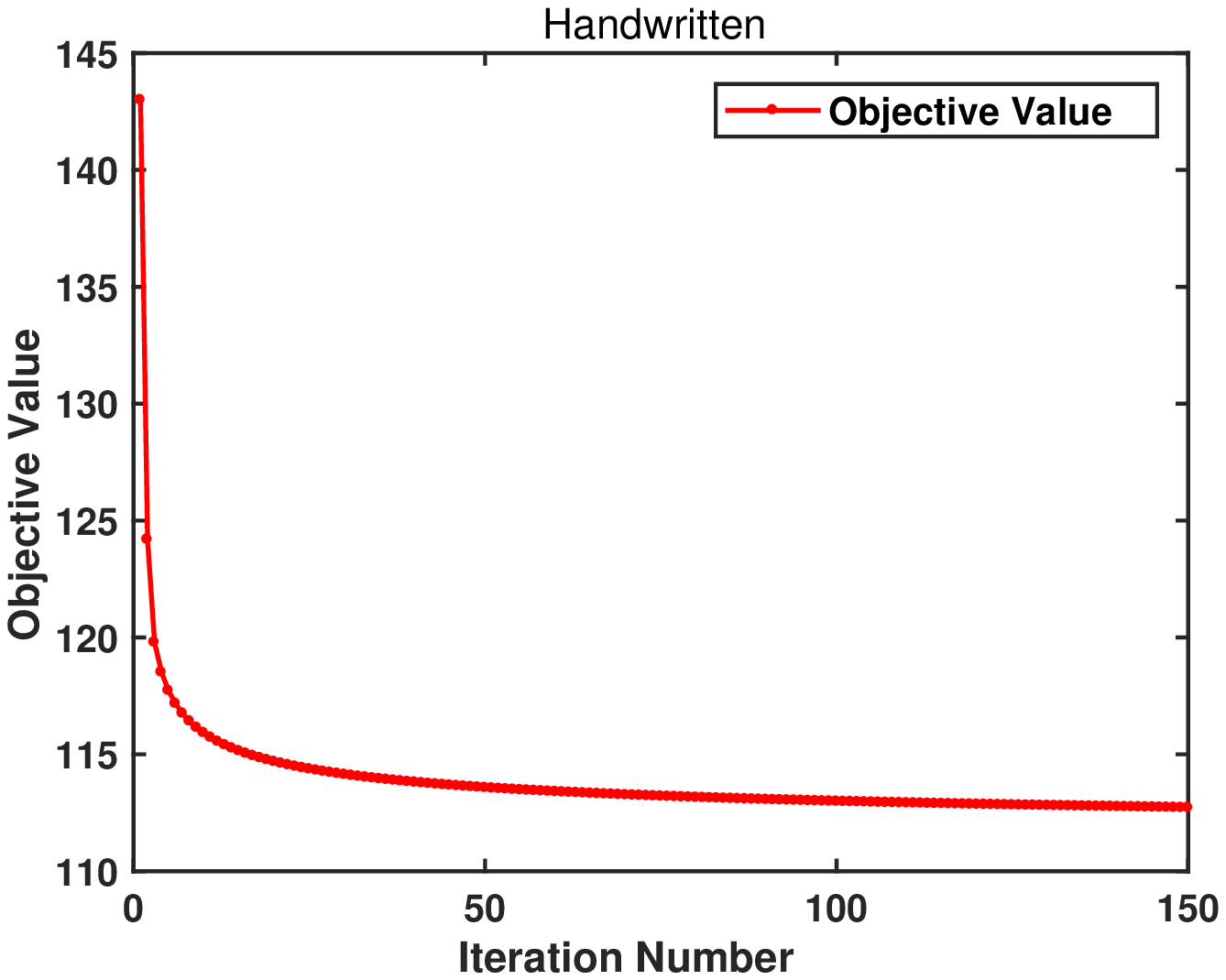}}
		\centerline{(e) HW}
	\end{minipage}
	\begin{minipage}{0.32\linewidth}
		\vspace{3pt}
		\centerline{\includegraphics[width=\textwidth]{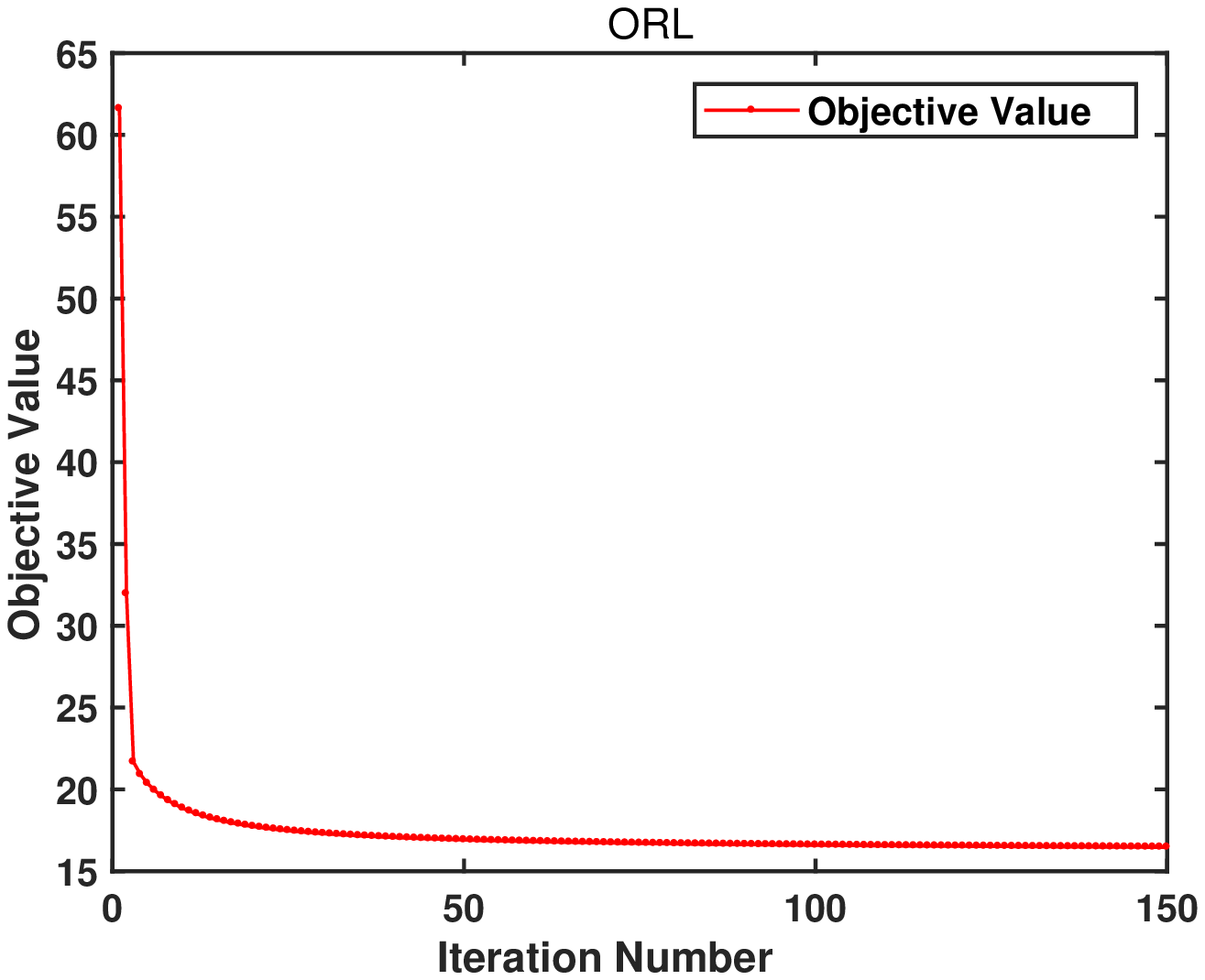}}
		\centerline{(f) ORL}
	\end{minipage}
	\caption{The convergence of the proposed method on 3Sources, BBC, BBCSport, CiteSeer, HW and ORL.}
	\label{obj}
\end{figure*}

\begin{figure*}
	
	\begin{minipage}{0.32\linewidth}
		\vspace{3pt}
		\centerline{\includegraphics[width=\textwidth]{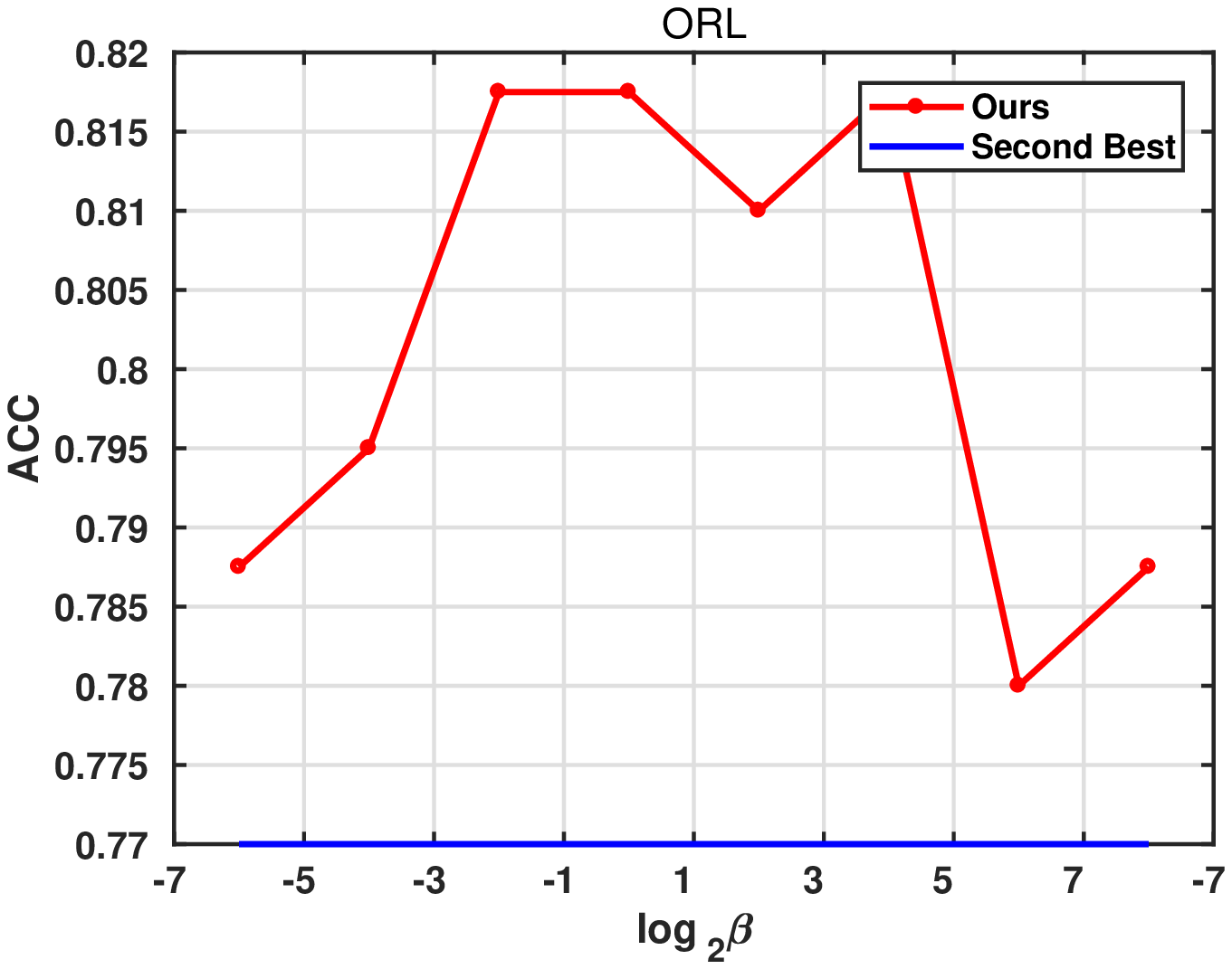}}
		\centerline{(a) ORL:$\beta$}
	\end{minipage}
	\begin{minipage}{0.32\linewidth}
		\vspace{3pt}
		\centerline{\includegraphics[width=\textwidth]{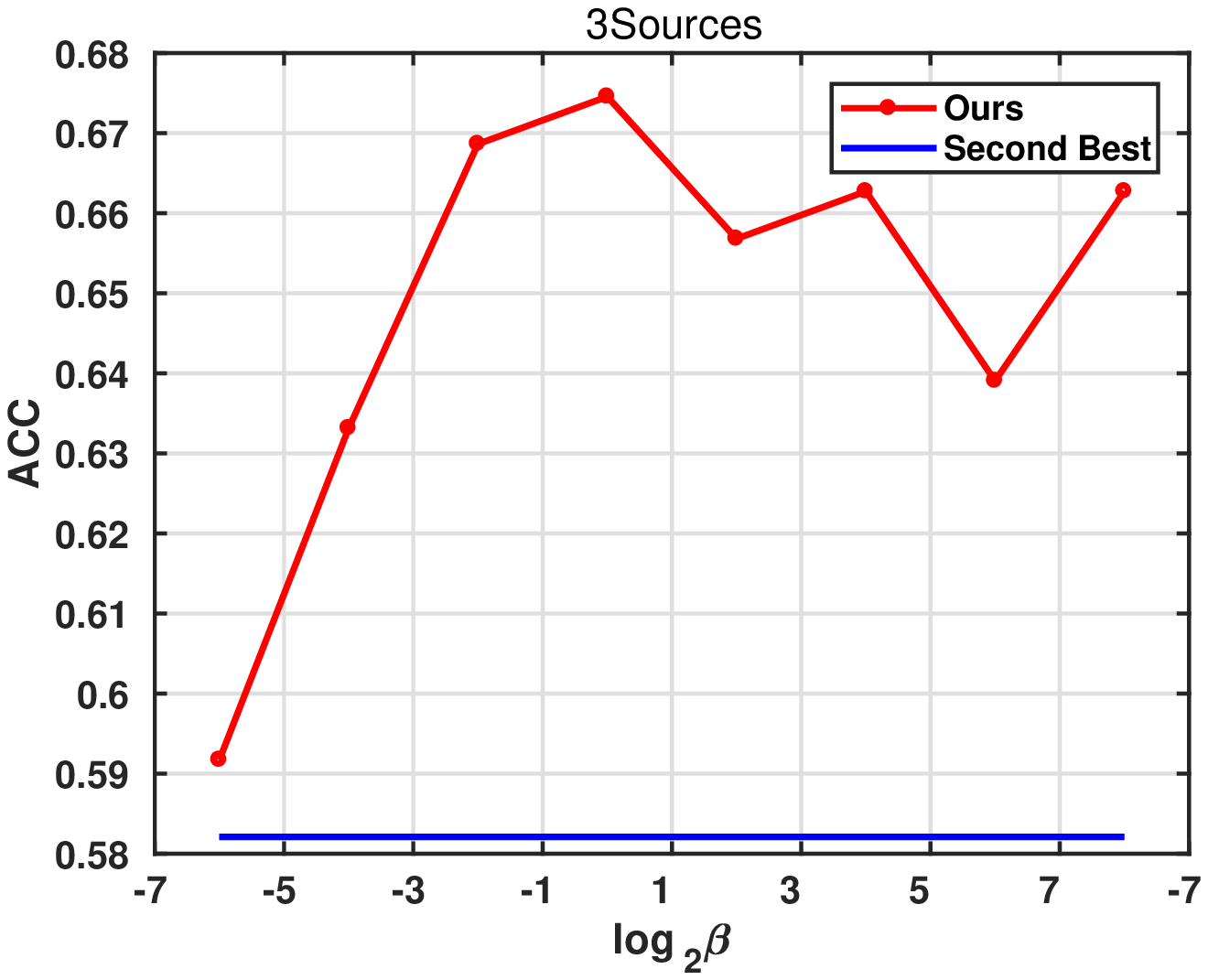}}
		\centerline{(b) 3Sources:$\beta$}
	\end{minipage}
	\begin{minipage}{0.32\linewidth}
		\vspace{3pt}
		\centerline{\includegraphics[width=\textwidth]{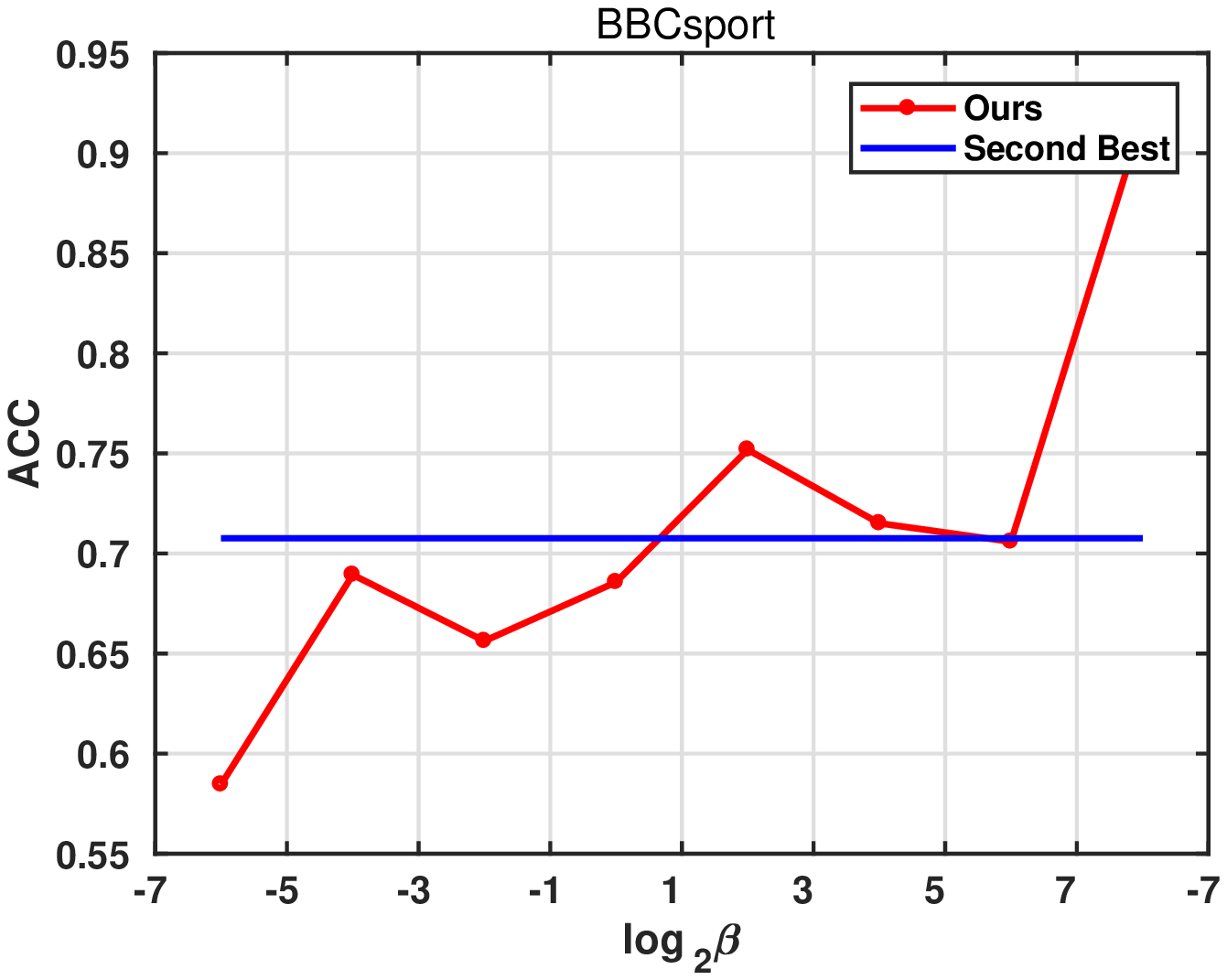}}
		\centerline{(c) BBCSport:$\beta$}
	\end{minipage}
	\caption{The sensitivity of the proposed method with the variation of $\beta$ on ORL, 3Sources and BBCSport.}
	\label{sensityBeta}
\end{figure*}

\begin{figure*}
	
	\begin{minipage}{0.32\linewidth}
		\vspace{3pt}
		\centerline{\includegraphics[width=\textwidth]{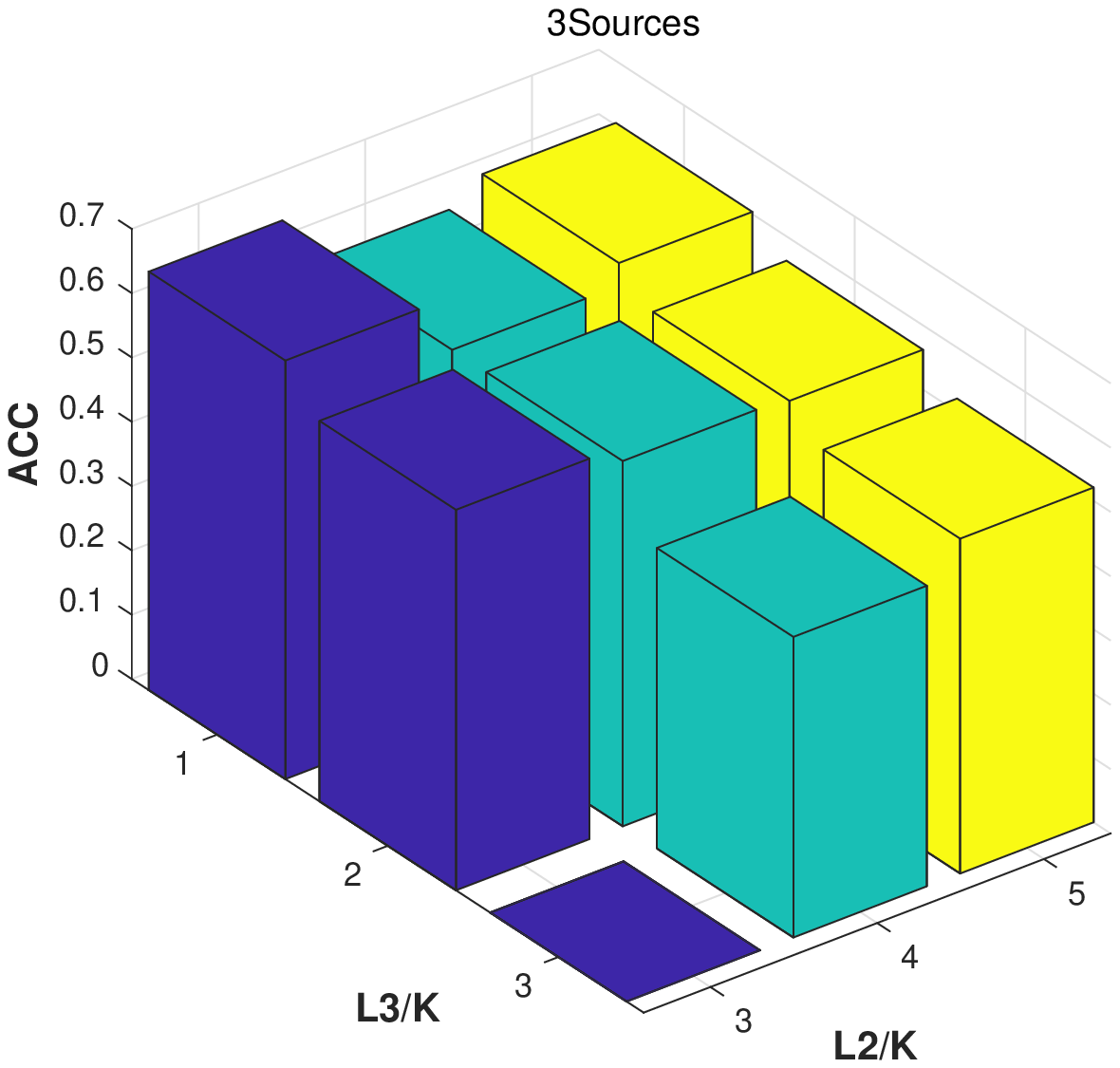}}
		\centerline{(a) 3Source:$l_2$ and $l_3$}
	\end{minipage}
	\begin{minipage}{0.32\linewidth}
		\vspace{3pt}
		\centerline{\includegraphics[width=\textwidth]{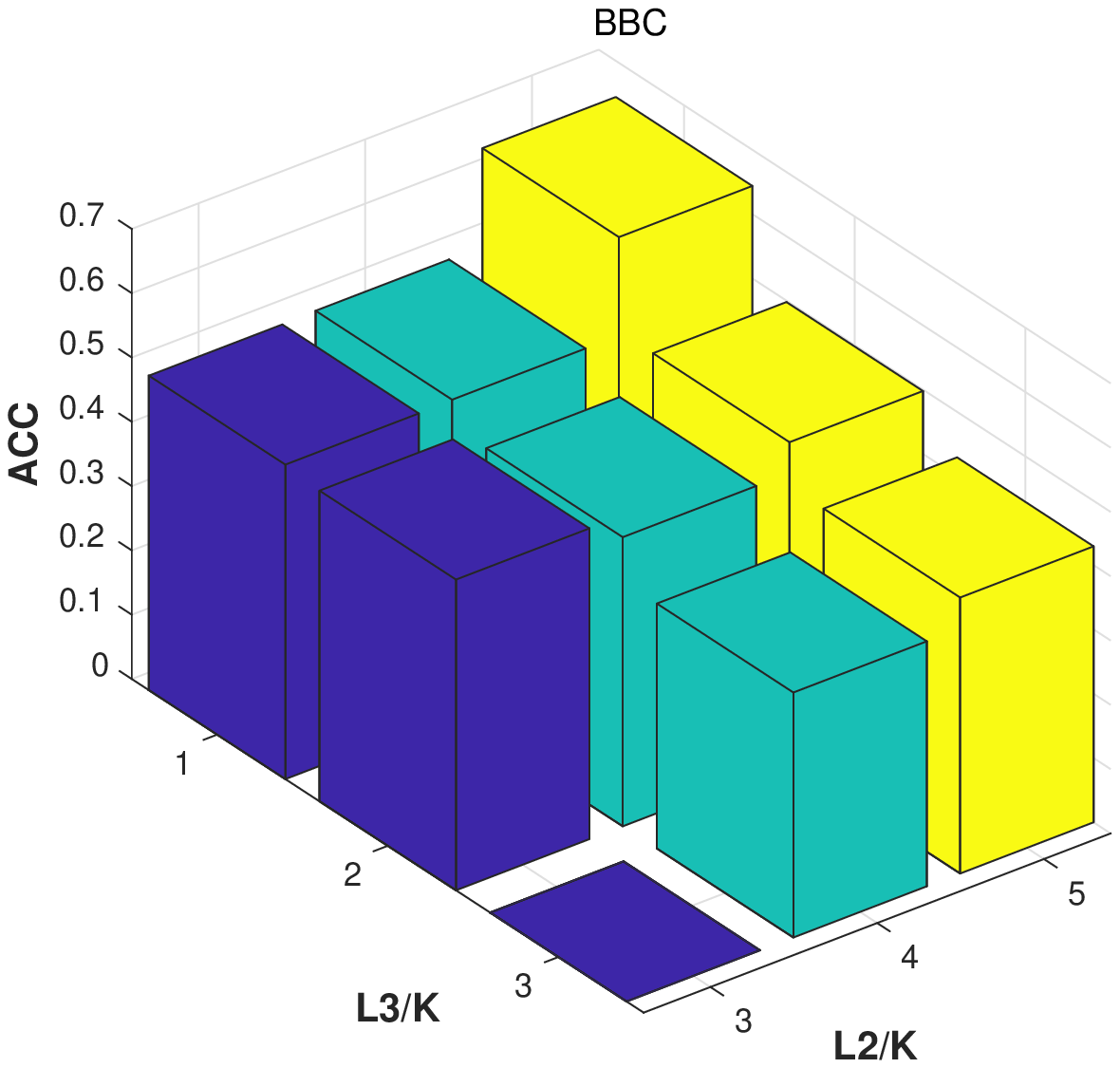}}
		\centerline{(b) BBC:$l_2$ and $l_3$}
	\end{minipage}
	\begin{minipage}{0.32\linewidth}
		\vspace{3pt}
		\centerline{\includegraphics[width=\textwidth]{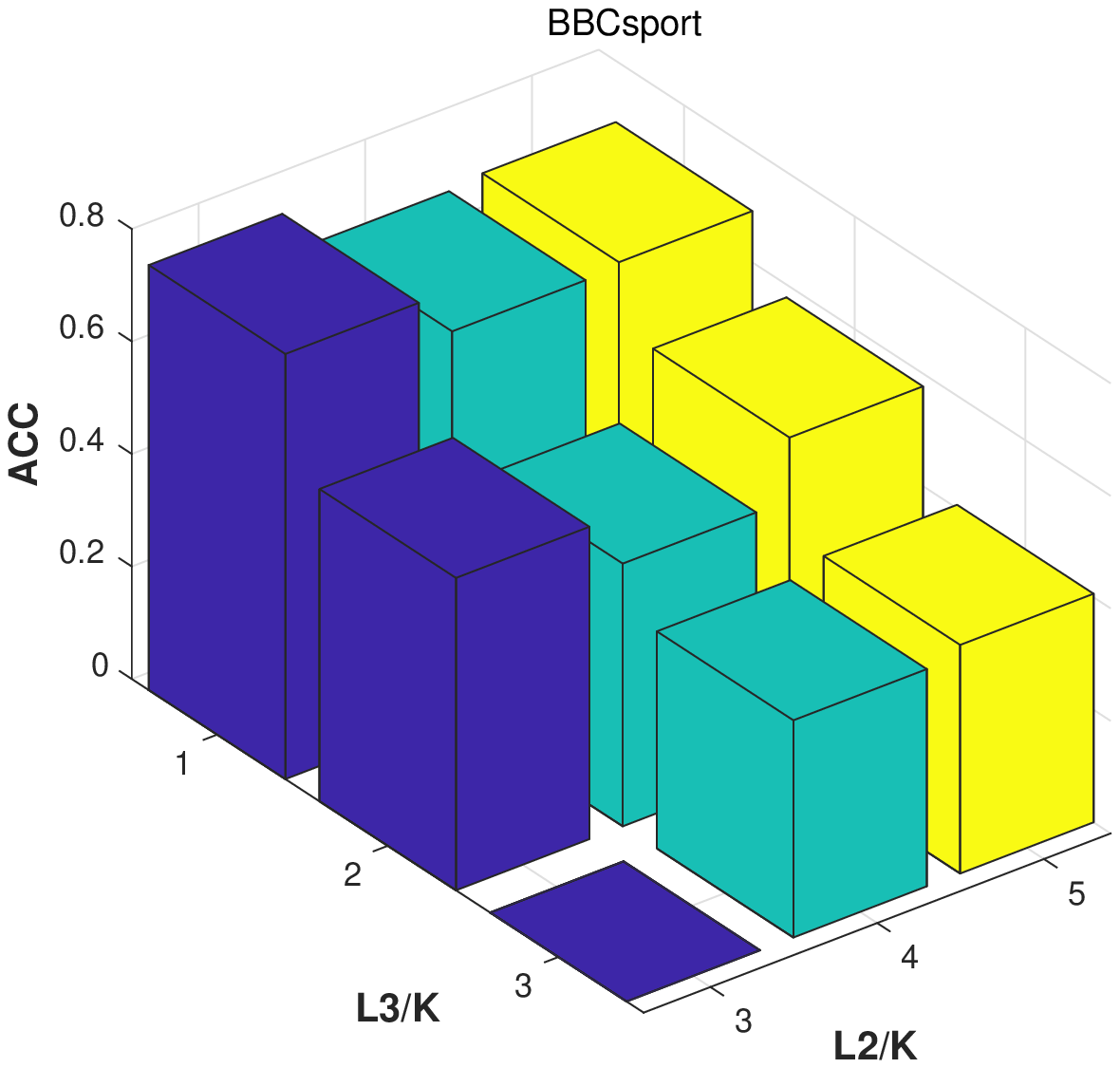}}
		\centerline{(c) BBCSport:$l_2$ and $l_3$}
	\end{minipage}
	
	\caption{The sensitivity of the proposed method with the variation of $l_2$ and $l_3$ in $p_3$ on 3Source, BBC and BBCSport.}
	\label{sentivityL2L3}
\end{figure*}

\begin{figure*}
	\centering
	\begin{minipage}[t]{0.32\textwidth}
		\centering
		\includegraphics[width=\textwidth]{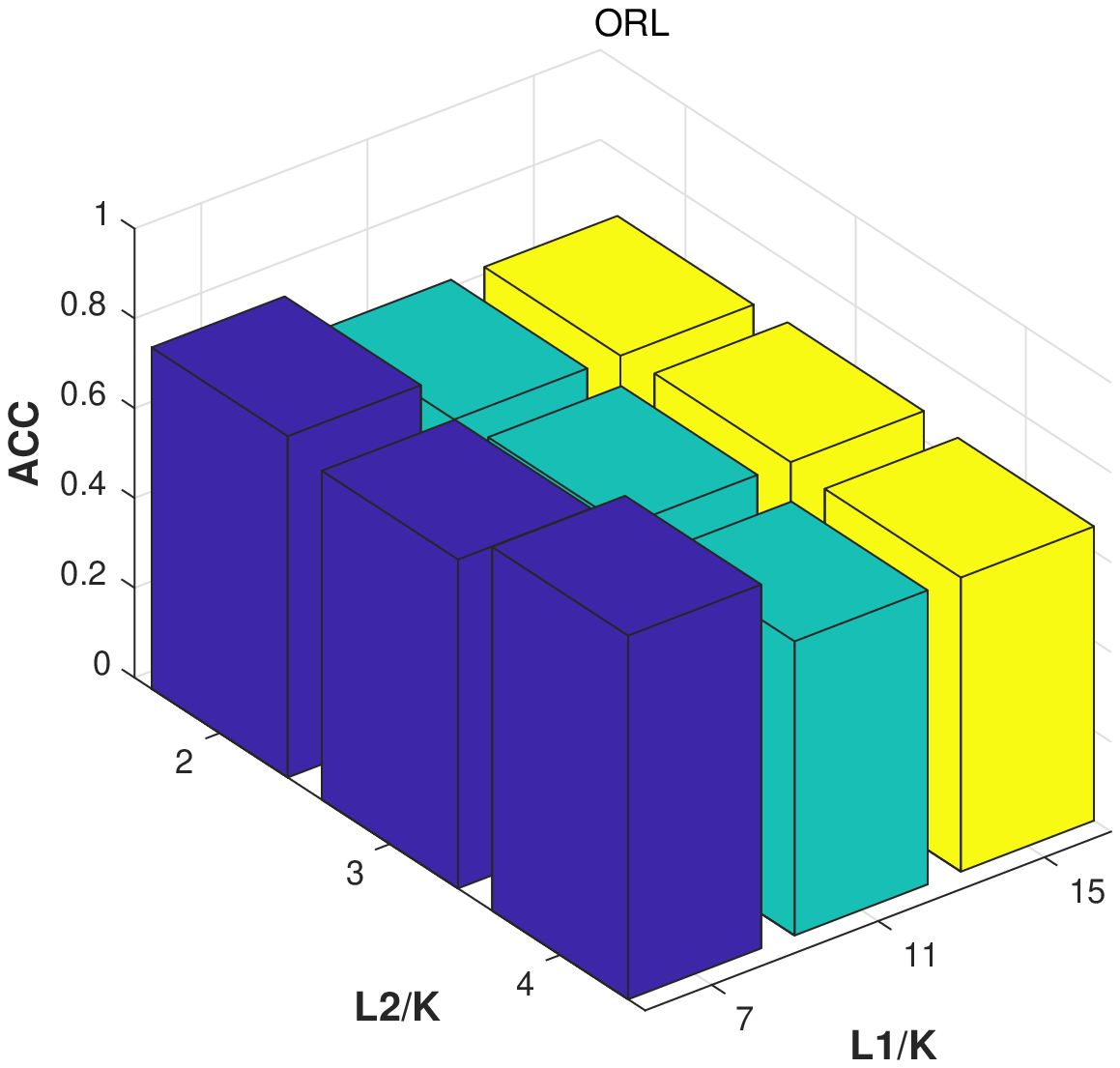}
		\centerline{(a) ORL:$l_1$ and $l_2$}
	\end{minipage}
	\begin{minipage}[t]{0.32\textwidth}
		\centering
		\includegraphics[width=\textwidth]{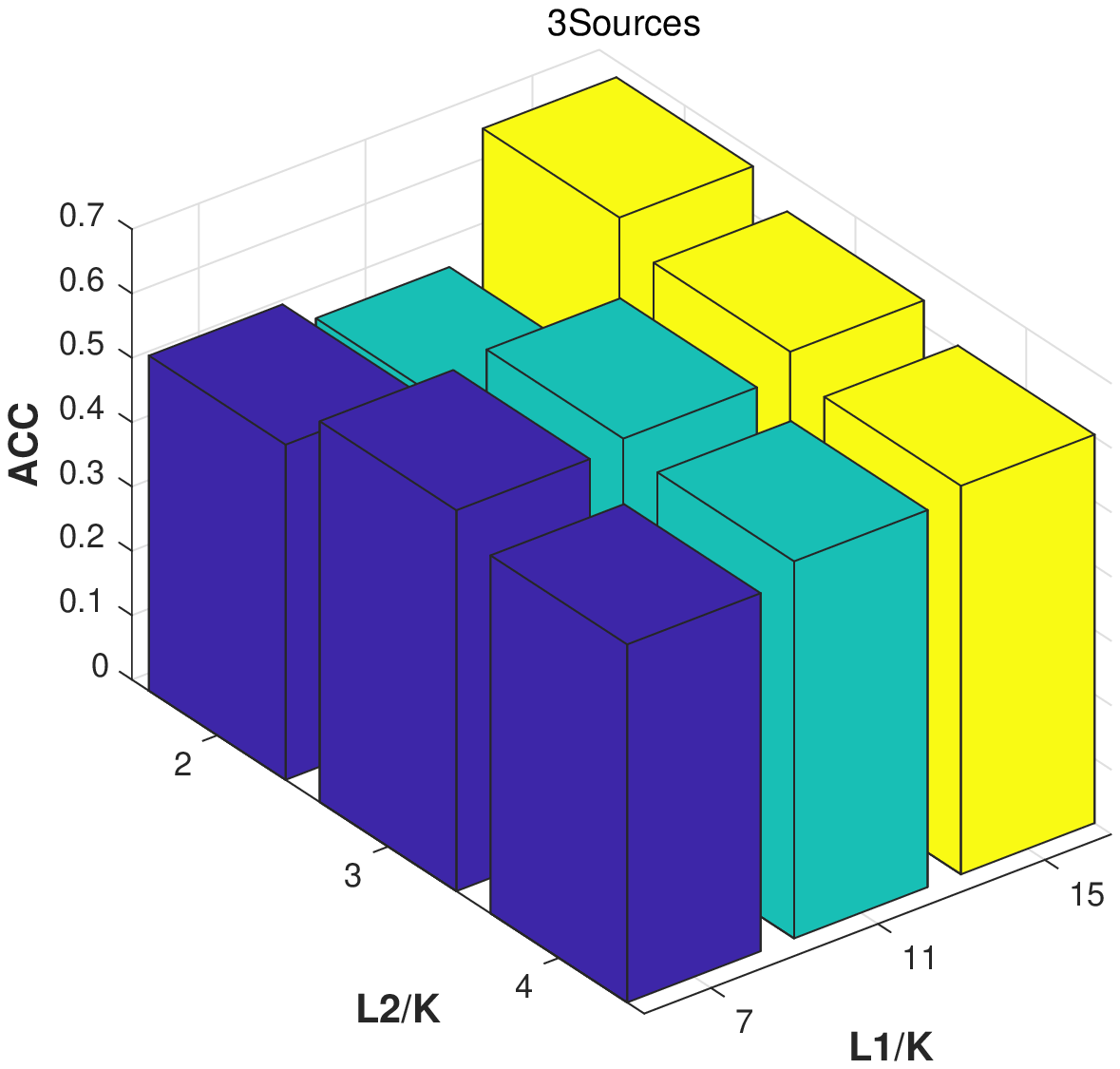}
		\centerline{(b) 3Sources:$l_1$ and $l_2$}
	\end{minipage}
	\begin{minipage}[t]{0.32\textwidth}
		\centering
		\includegraphics[width=\textwidth]{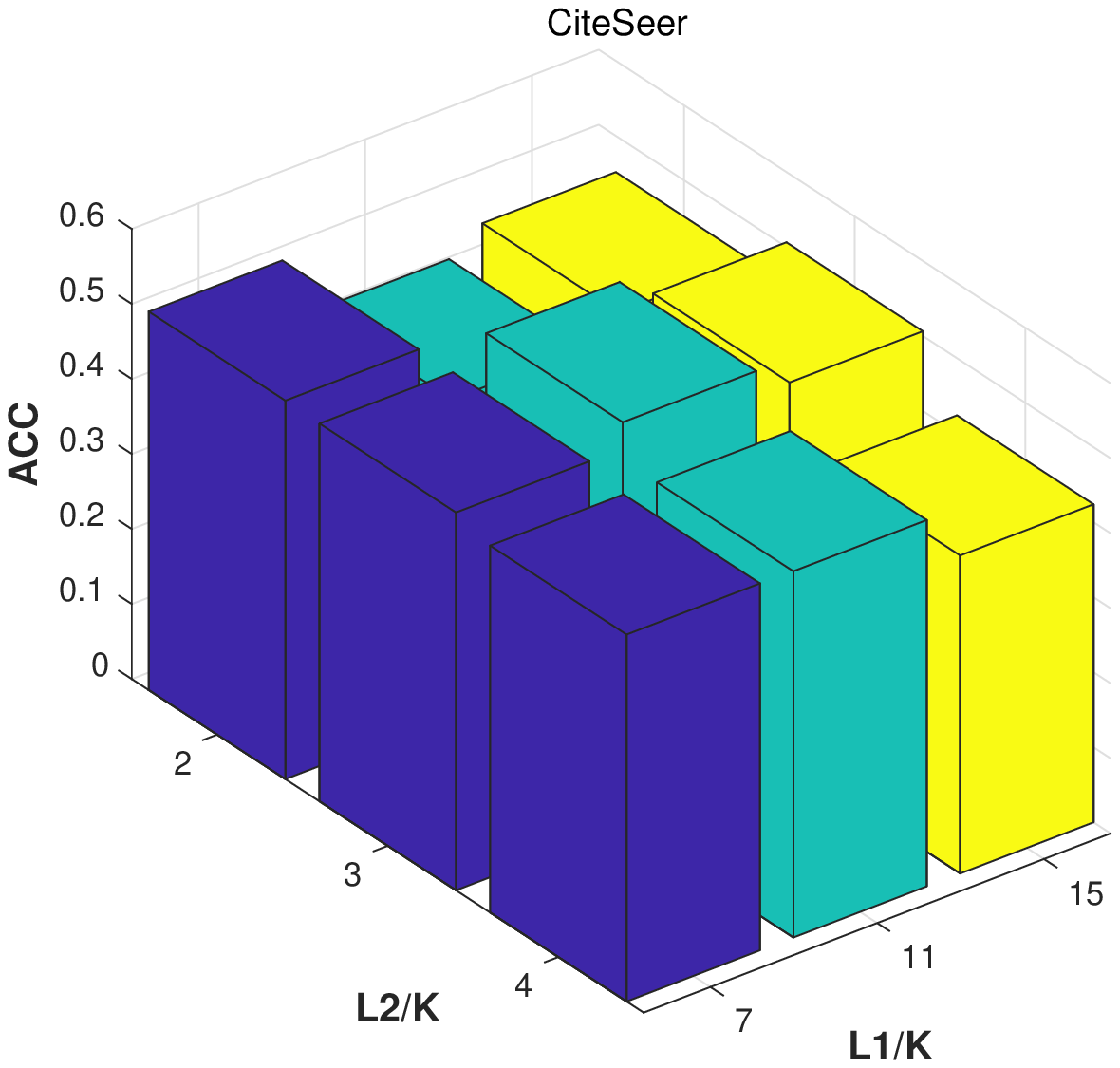}
		\centerline{(c) CiteSeer:$l_1$ and $l_2$}
	\end{minipage}
	\caption{The sensitivity of the proposed method with the variation of $l_1$ and $l_2$ in $p_3$ on ORL, 3Sources and CiteSeers.}
	\label{sentivityL1L2_3}
\end{figure*}

\subsection{Ablation Study}

\begin{table}[ht]
	\caption{ACC of different layers on six benchmark datasets}
	\label{ACC_DL1}
	\begin{tabular}{lllllll}
		\toprule
		Depth &BBCSport &3Sources &BBC &CiteSeer &ORL &HW   \\
		\midrule
		$\left[k\right]$       & 75.55          & 61.54          & 54.74          & 44.81          & 73.50          & 64.65          \\
		$\left[l_2,k\right]$   & 88.50          & 64.50          & 63.50          & 53.83          & 77.25          & 59.50          \\
		$\left[l_1,l_2,k\right]$ & \textbf{91.73} & \textbf{70.41} & \textbf{71.68} & \textbf{53.86} & \textbf{79.25} & \textbf{84.05} \\
		\bottomrule
	\end{tabular}
\end{table}

We have done a set of ablation experiments like the number of layers $p$ is one, two, and three. The purpose is to verify if the number of layers deeper, the hidden information can be easier to be extracted, and the more valuable representations can be learned.

We record the best parameters like $p_{3}=[a,b,k]$ when depth is third. As can see in Table \ref{ACC_DL1}, we compare the results of different depths like $p_{3}=[a,b,k]$, $p_{2}=[b,k]$, $p_{1}=[k]$. It is obvious to find that the results of three layers are always greater than two, and the results of two layers are always greater than one. It is easy to calculate the performance improvement on BBCSport, 3Sources, BBC, CiteSeer and ORL by $12.95\%$, $2.96\%$, $8.76\%$, $9.02\%$ and $3.75\%$ when the number of layers changes from $p_{1}$ to $p_{2}$, and the performance improvement by $3.23\%$, $2.96\%$, $8.18\%$, $0.03\%$ and $3.75\%$ when the number of layers changes from $p_{2}$ to $p_{3}$. So it is very necessary to choose an appropriate number of layers for all datasets.

\subsection{Visualization}

We visualize the clustering results of our algorithm and the comparison diagram of our algorithm, DMVC, and AwDMVC in Figure \ref{VisualizationDifferentIteration} and Figure \ref{VisualizationDifferentMethod} respectively. In these figures, we represent the samples of the same class as a color. The points of the same color become closer and the points of different colors become further, the better the clustering performance is. It can be seen from Figure \ref{VisualizationDifferentIteration} that the differences between intra-class structure and inter-class are becoming more and more obvious with the increase of the number of iteration on datasets BBCSports, BBC, and HW. It can be seen from Figure \ref{VisualizationDifferentMethod} that the clustering effect of DMVC is least obvious. For DMVC, some intra-class structures can be seen, while the boundaries between clusters are very vague or even not. AwDMVC clusters into a ring structure for one class, but a ring always contains other classes, which greatly reduces the clustering effect. In contrast, clear clustering structures can be seen in our algorithm on two benchmarks.

\subsection{Convergence}
It is theoretically guaranteed that our algorithm converges to a local minimum. We also conduct experiments to verify that the algorithm is convergent or not. As shown in Figure \ref{obj}, the objective value curves are plotted in red on datasets 3Source, BBC, and BBCSport. The experimental results prove that our proposed algorithm can decrease monotonically and the iterations less than 150 usually. Thus it experimentally proves the convergence of our algorithm.

\subsection{Parameter Sensitivity}
There are two sets of parameters in our proposed method, i.e., the balance coefficient $\beta$ and the size of layer $p$. Next we will analyze the sensitivity of $\beta$, the selection rules of the last parameter in $p_2$ and $p_3$, and the sensitivity of $l_1$ and $l_2$ in $p_3$.

\subsubsection{Sensitivity of $\beta$}

Figure \ref{sensityBeta} shows the influence of ACC result concerning the parameter $\beta$ under the best layer size setting $p$ which is obtained in previous experiments. Each small picture in Figure \ref{sensityBeta} contains two curves, where the red one is ours, the blue one is the second best algorithm. We can find that our algorithm outperforms the second best algorithm in most range of the $\beta$ on most of the benchmarks even if it is a little sensitive to the parameter $\beta$.





\subsubsection{Sensitivity of $l_1$ and $l_2$ in $p_3$}

The Figure \ref{sentivityL1L2_3} shows the sensitivity experimental results of $l_1$ and $l_2$ in $p_3$ on ORL, 3Source and CiteSeer. From these figures, we observe that it is relatively stable in most parameter combinations without the above-mentioned discipline. Despite slight variation, it outperforms most algorithms in most of the benchmarks.

\section{Conclusion} 
\label{conclusion}

In this paper, we propose a novel multiple clustering framework with deep semi-NMF, which simultaneously optimizes deep representation learning and consensus graph constructing. In other words, the deep representation can be refined by the global consensus graph and vice versa. Through the multi-layer projection and the guidance of a consensus geometric structure that is constrained by a graph, the representation learned can contain more hidden attributes of the original features. Extensive experiments are conducted on six benchmarks, demonstrating the effectiveness of our proposed algorithm by comparing with ten SOTA methods. In the future, we will consider learning a consensus representation with a rotation matrix directly and construct the consensus graph more discriminative.

\section*{Acknowledgment}

This work was supported by the National Key R\&D Program of China 2018YFB1003203 and the National Natural Science Foundation of China (Grant NO.61672528, NO. 61773392, NO. 61872377).


%

\ifCLASSOPTIONcaptionsoff
  \newpage
\fi

\bibliographystyle{IEEEtran}
\bibliography{MVC-DMF-GGR}

\end{document}